\pdfoutput=1
\documentclass[letterpaper,twocolumn]{article}
\usepackage[submission]{aaai2027}
\usepackage[hyphens]{url}
\usepackage{graphicx}
\usepackage{caption}
\usepackage{amsmath,amssymb,amsthm}
\usepackage{booktabs}
\usepackage{multirow}
\usepackage{pifont}
\newtheorem{theorem}{Theorem}
\newtheorem{lemma}[theorem]{Lemma}
\newtheorem{proposition}[theorem]{Proposition}
\newtheorem{corollary}[theorem]{Corollary}
\newtheorem{assumption}[theorem]{Assumption}
\newtheorem{remark}[theorem]{Remark}

\renewcommand{\floatpagefraction}{0.6}

\usepackage{enumitem}
\makeatletter
\def\th@plain{\thm@headfont{\bfseries}\itshape}
\makeatother

\usepackage{titlesec}
\titlespacing*{\section}{0pt}{0.7ex plus .2ex minus .2ex}{0.4ex plus .1ex}
\titlespacing*{\subsection}{0pt}{0.6ex plus .2ex minus .2ex}{0.3ex plus .1ex}

\newcommand{\std}[1]{{\scriptsize$\pm$#1}}

\title{Active Spiking Perception: The Membrane Potential as a Belief State\\ for Anytime 3D Point Cloud Recognition}
\author{Akarsh Jain\textsuperscript{1}, \ Arya Pawa\textsuperscript{2}, \ Ayush Debnath\textsuperscript{3}, \ Smera Rawal\textsuperscript{4}, \ Sayeed Shafayet Chowdhury\textsuperscript{5}}
\affiliations{%
\textsuperscript{1}Indian Institute of Technology Indore \quad
\textsuperscript{2}Thadomal Shahani Engineering College, University of Mumbai \quad
\textsuperscript{3}Indian Institute of Technology Kharagpur\\
\textsuperscript{4}Vellore Institute of Technology Chennai \quad
\textsuperscript{5}Indiana University Indianapolis\\[2pt]
\texttt{ee240002007@iiti.ac.in}, \ \texttt{aryapawa903@gmail.com}, \ \texttt{ayush.d@kgpian.iitkgp.ac.in},\\
\texttt{smerarawal@gmail.com}, \ \texttt{chowdh23@purdue.edu}}

\begin{document}
\maketitle
\maketitle

\begin{abstract}
Spiking point cloud networks usually scan space in a fixed, input-agnostic order, which leaves the most distinctive resource of spiking computation, the temporal evolution of the membrane potential, unused as a locus of decision-making. \emph{Active Spiking Perception} (ASP) recasts 3D recognition as an iterative decision process in which the network's own leaky integrate-and-fire (LIF) membrane potential, read as a running belief over the class, selects the next spatial chunk to observe and triggers confidence-margin early exit. A lightweight Slice-Selection Policy scores unvisited farthest-point-sampled chunks from the membrane state and precomputed geometric descriptors, trains end-to-end through a straight-through Gumbel--Softmax, reduces to an argmax at inference, and adds about 2\% of backbone parameters. We prove that leaky integration \emph{is} the recursive log-posterior update of a Bayesian filter, exactly for an idealised accumulator and approximately for the trained network, that the exit rule attains distribution-free selective risk with no multiple-testing penalty at the stopping time, and that streaming state carry-forward is exactly equivalent to prefix recomputation with geometrically bounded finite-precision drift. ASP reaches 90.62\% and 93.28\% on ModelNet40 and ModelNet10, 1.7 points below the strongest spiking baseline on ModelNet40 at a larger backbone, while adding a certified anytime interface no baseline offers. The same mechanism transfers without architectural change to dense prediction, giving 83.21 instance mIoU on ShapeNetPart and 48.50 mIoU on S3DIS Area~5, to our knowledge the first spiking results on S3DIS Area~5, and, fixation replacing chunk selection, to a foveated \emph{non-spiking} transformer, so the policy is not tied to spiking backbones: cost is exactly linear in observations and the threshold is a \emph{measured} compute dial spanning $2.8\times$ to $1.35\times$ less energy against a control matched to within 768 parameters, which nonetheless stays 4.52 points ahead. One limitation is concrete: one S3DIS class is unidentifiable at the crop size we use, and we give the prediction that would fix it.
\end{abstract}

\section{Introduction}
Three-dimensional point cloud understanding underpins embodied applications where decisions must be accurate, fast and cheap on constrained hardware. Accuracy has advanced steadily, from PointNet \citep{qi2017pointnet} to Point Transformer \citep{zhao2021point}, at proportionally rising floating-point cost. Spiking neural networks (SNNs) offer an alternative substrate, replacing dense multiply--accumulate with sparse, event-driven accumulate-only operations that map onto neuromorphic hardware such as Loihi \citep{davies2018loihi} at roughly a fourth of the energy per operation \citep{lemaire2022analytical}. Surrogate gradients \citep{fang2021incorporating,neftci2019surrogate} made deep SNNs trainable, and Spiking PointNet \citep{ren2023spiking}, SPT \citep{wu2025spt} and SPM \citep{wu2025spm} have narrowed the gap to ANN baselines.

Yet all inherit a choice from their ANN predecessors that sits uneasily with spiking computation: the cloud is partitioned into local regions and the \emph{entire} set is processed at every timestep in a fixed order. For a feedforward ANN that is reasonable, a ``timestep'' there being bookkeeping. For an SNN, whose membranes encode an incrementally refined belief, a fixed order wastes effort twice: discriminative structure is spatially non-uniform and class-dependent, so uniform allocation ignores which regions matter, and once evidence is decisive further regions burn energy for nothing. Early-exit methods \citep{teerapittayanon2016branchynet,huang2018msdnet,graves2016act} ask \emph{when} to stop across depth, and adaptive token selection \citep{baiocchi2024adapt,rao2021dynamicvit} prunes \emph{after} a first encoding pass, but neither question has been posed for the spiking 3D setting. Hence our central question: \emph{can a spiking network decide, from its own membrane state, where in the input to look next?}

\begin{figure*}[!t]
\centering
\includegraphics[width=0.92\textwidth]{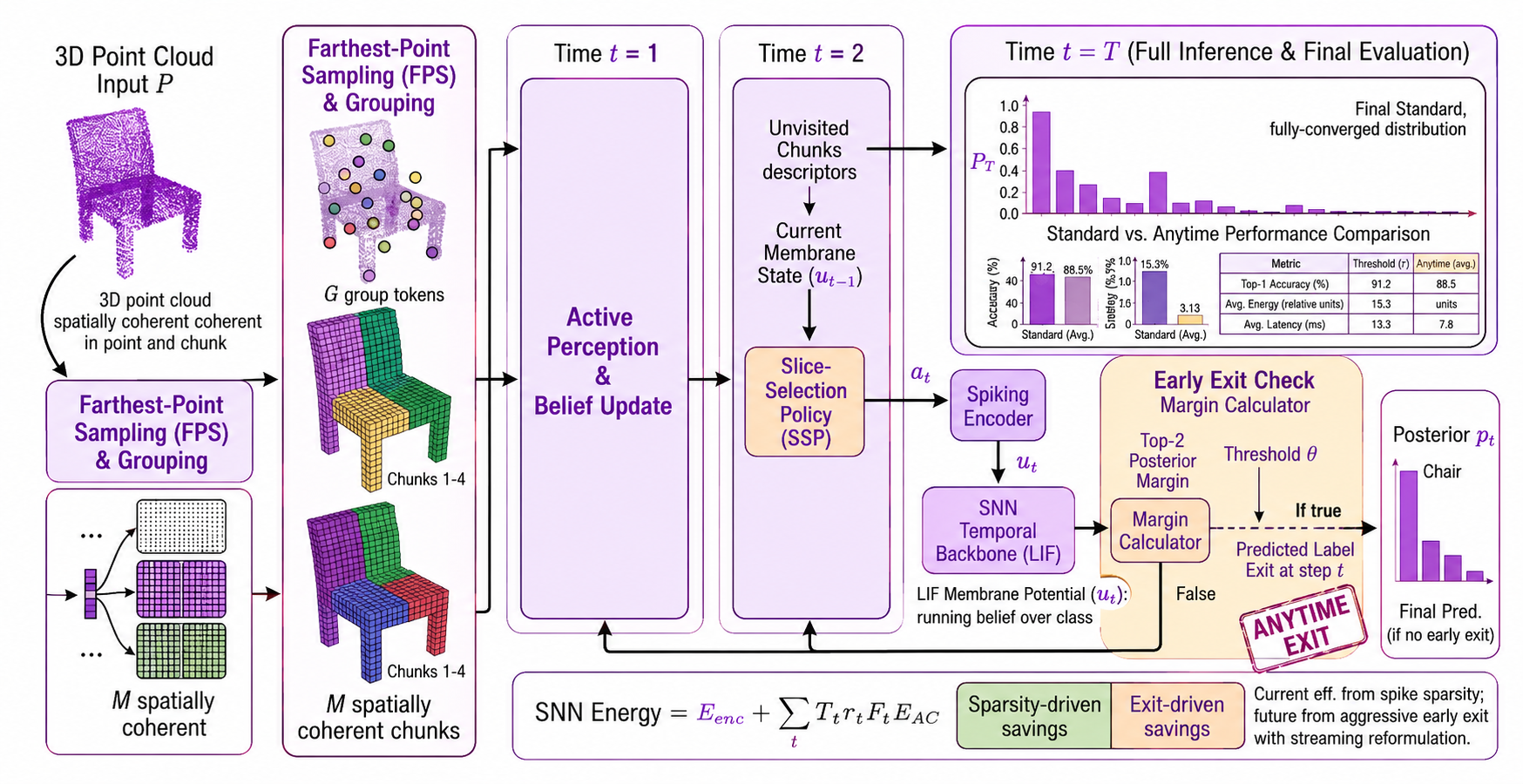}
\caption{ASP architecture. Preprocessing (left) runs once: FPS groups the cloud into $M$ chunks with parameter-free descriptors. The loop (centre) runs at most $T$ steps: the SSP reads membrane $u_{t-1}$ and unvisited-chunk descriptors to pick $a_t$, the backbone updates the membrane, and the top-2 margin is checked against the calibrated $\theta$. Right: converged posterior against the anytime prediction at exit.}
\label{fig:arch}
\vspace{-5mm}
\end{figure*}

We answer it with \emph{Active Spiking Perception} (ASP), in which observation is an iterative, learned decision process rather than a fixed scan (Figure~\ref{fig:arch}). The conceptual move is that the LIF membrane potential is already a compact summary of everything observed so far. Every SNN maintains it for free, and it suffices to decide what to observe next. A lightweight Slice-Selection Policy reads it with cheap precomputed descriptors of unvisited regions and selects where to look; a confidence rule stops once the prediction is no longer expected to change. The temporal dimension then performs adaptive observation instead of fixed rescanning, and simple inputs cost fewer timesteps than ambiguous ones.

We are deliberate about what is claimed. The membrane is a usable controller, its guarantees hold at a measured risk level, and every mechanism the theory posits can be caught behaving as predicted. It transfers to part and scene segmentation, and to foveated vision, where fixation replaces chunk selection inside an ordinary ANN, so the controller is not a spiking-only construction.

\paragraph{Contributions.} \textbf{(i)} The first formulation of 3D point cloud understanding as active, sequential observation \emph{native to a spiking backbone}. \textbf{(ii)} The Slice-Selection Policy, scoring unvisited regions from membrane state and offline geometry at about 2\% parameter overhead. \textbf{(iii)} Guarantees on Bayesian sufficiency of the membrane, anytime-valid selective risk, and exact streaming equivalence with a quantization-drift bound, each tied to a property we \emph{measure}. \textbf{(iv)} An energy accounting separating spike sparsity from early exit, stated as a hardware-parameter inequality. \textbf{(v)} Evidence the policy is not spiking-specific: the same scorer and exit rule, fixation replacing chunk selection, run unchanged inside a \emph{conventional ANN} transformer, spanning $2.83\times$ to $1.35\times$ less energy against a 768-parameter-matched control. \textbf{(vi)} Evaluation across classification, part and scene segmentation and that foveated model, including, to our knowledge, the first spiking results on S3DIS Area~5.

\section{Related Work}
Dense point cloud networks \citep{qi2017pointnet,qi2017pointnetpp,wang2019dgcnn,zhao2021point,yu2022pointbert,ma2022pointmlp} refine \emph{how} a cloud is represented, never \emph{when} or \emph{where}. Spiking 3D models \citep{ren2023spiking,qiu2025e3dsnn,wu2025spt,wu2025spm}, which already report part and outdoor semantic segmentation, have concentrated on matching ANN accuracy at all, so each processes all partitions in a fixed order at every timestep, leaving membrane accumulation unused as a scheduling signal. ASP repurposes that state as a \emph{controller}. The closest efficiency-oriented line is adaptive token selection, where AdaPT \citep{baiocchi2024adapt} prunes tokens inside a point transformer under a runtime budget, token-pruning vision transformers \citep{rao2021dynamicvit,liang2022evit} do the same for images, and adaptive-computation methods \citep{teerapittayanon2016branchynet,huang2018msdnet,graves2016act} halt across depth. Three differences separate ASP: those methods decide what to \emph{discard} after encoding every token once, whereas ASP decides what to \emph{acquire}, so no unselected region is encoded at all; its controller is an internal recurrent state the network already maintains; and selection is coupled to a \emph{certified} stopping rule, so every intermediate answer carries a distribution-free risk level. Among glimpse models \citep{bajcsy1988active,ballard1991animate,gregor2015draw,jonnalagadda2021foveater}, RAM \citep{mnih2014ram} is the closest ancestor, but its location network reads a continuous LSTM state trained by policy gradient, where ASP scores chunks from a LIF membrane through a Gumbel--Softmax relaxation, keeping inference accumulate-only. Our guarantees build on selective classification \citep{geifman2017selective}, the $(1-1/e)$ greedy bound \citep{nemhauser1978analysis,krause2005near} as distinct from adaptive submodularity \citep{golovin2011adaptive}, and conformal calibration \citep{guo2017calibration,angelopoulos2024crc}; CALM \citep{schuster2022calm} halts along \emph{depth} over a fixed input where we halt along \emph{acquisition} (supplementary~\S\,C).

\section{Active Spiking Perception}
\paragraph{Configuration up front.} Unless stated otherwise, results use $M{=}4$ chunks of $K{=}32$ tokens and $T\le M$ steps, where the loop runs near full utilisation, so spike sparsity dominates the measured energy reduction. $M$ is the knob that matters most, and the separately trained $M{=}16$ model reported later is where adaptive observation becomes the larger effect.

\paragraph{Problem formulation.} Let $P\in\mathbb{R}^{N\times3}$ be an input cloud. Farthest-point sampling selects $G$ centroids; a ball query about each collects $K$ points; the $G$ group tokens are partitioned into $M$ spatially coherent chunks, which may overlap and need not cover $P$ (coverage is a design parameter, not a constraint). We treat observation as a $T$-step decision process, $T\le M$: writing $V_t$ for chunks observed up to step $t$ with $V_0=\emptyset$, the model (i) selects $a_t\notin V_{t-1}$ under a policy conditioned on the current spiking state; (ii) updates the LIF membrane via that chunk's tokens; (iii) emits $\hat{y}_t$, sets $V_t=V_{t-1}\cup\{a_t\}$; and (iv) halts if a margin criterion is met. Encoder, policy and halting rule train jointly against one objective.

\subsection{Chunk Encoder, Spiking Backbone, and Policy}
\label{sec:head}
Relative coordinates per group are embedded by an EdgeConv block \citep{wang2019dgcnn} on a \emph{static} $k$-NN graph, $e_m=\max_{i\in G_m}\psi(h_i)\in\mathbb{R}^{D}$; static so the descriptors below stay consistent with the neighbourhoods the encoder sees, satisfying supplementary Proposition~S15 by construction. For layer $l$ at step $t$ with membrane $u^{(l)}_t$, spikes $s^{(l)}_t$, per-neuron leak $\lambda^{(l)}$ and threshold $\theta^{(l)}$,
\begin{align}
u^{(l)}_t&=\lambda^{(l)}\!\odot u^{(l)}_{t-1}+\mathrm{ReLU}\big(\mathrm{BN}(W^{(l)}x^{(l-1)}_t)\big)-\theta^{(l)}\!\odot s^{(l)}_{t-1},\label{eq:lif2}\\
s^{(l)}_t&=\Theta\big(\mathcal{N}(u^{(l)}_t)-\theta^{(l)}\big).\label{eq:lif3}
\end{align}
Eq.~\eqref{eq:lif2} is a \emph{soft reset}, which keeps the overshoot, the evidence separating a marginal from an emphatic spike and the signal the policy consumes. Leak and threshold follow DIET-SNN \citep{rathi2021diet}, $\mathcal{N}$ is membrane batch normalisation \citep{guo2023membrane} on the firing path only so the belief is never rescaled by batch statistics, and training uses the arctangent surrogate \citep{fang2021incorporating}. After $a_t$ is selected, a selective-scan mixer \citep{gu2023mamba} over the observed prefix drives $L$ residual LIF cells, with belief $b_t=\mathrm{LayerNorm}(u^{(L)}_t)$ taken from the membrane and \emph{not} the residual sum, which mixes an analog shortcut.

Each chunk is summarised by a parameter-free descriptor computed once from coordinates,
\begin{equation}
g_m=\big[\bar{x}_m;\,\mathrm{Var}[x_m];\,\max_{i\in S_m}\|x_i-\bar{x}_m\|_2;\,\|\bar{x}_m\|_2\big]\in\mathbb{R}^{8},
\label{eq:descriptor}
\end{equation}
which says \emph{where} a region is without revealing \emph{what} it contains, since resolving content is what the observation step is for. The policy scores unvisited chunks bilinearly, $q_{t,m}=w^{\top}\tanh(W_ub_{t-1}+W_gg_m)$ for $m\notin V_{t-1}$, with $W_u,W_g,w$ its only parameters and visited chunks masked out. Since $W_gg_m$ is geometry-only and hoisted out of the loop, per-step scoring costs $O((M{+}D)d_{\mathrm{ssp}})$, negligible against the backbone. Training draws a hard Gumbel--Softmax sample \citep{jang2017categorical} under geometric annealing, one-hot forward so the training graph matches inference semantics, softened backward; at inference $a_t=\arg\max_m q_{t,m}$. Alternative \texttt{random} and \texttt{fps\_order} modes sit behind one flag sharing every other component, so any measured difference is attributable to the selection rule alone.

\paragraph{Energy pricing.} A multiply--accumulate collapses to an accumulate only when the presynaptic activation is binary, which holds at exactly three sites: the block input projection, its edge accumulation, and the LIF head. Everything else is priced at $E_{\mathrm{MAC}}$; supplementary Table~S2 applies the partition, and supplementary~\S\,G.6 restates the comparison as an inequality over hardware parameters.

\subsection{Certified Early Exit and Training}
Let $p_t=\mathrm{softmax}(\hat{y}_t)$ with top two entries $p^{(1)}_t\ge p^{(2)}_t$, margin $\Delta_t=p^{(1)}_t-p^{(2)}_t$, and $T_\theta=\min\{t\le M:\Delta_t>\theta\}$, taking $T_\theta{=}M$ if empty; the output averages logits up to $T_\theta$. We do not tune $\theta$, we \emph{calibrate} it against a target risk by the split-conformal procedure of Theorem~\ref{thm:conformal}. Training uses Temporal Efficient Training \citep{deng2022temporal}, weighting every step equally so intermediate states stay independently discriminative, which both logit averaging and the exit rule require:
\begin{equation}
\mathcal{L}=\frac{1}{M}\sum_{t=1}^{M}\mathrm{CE}(\hat{y}_t,y)+\frac{\lambda_{\mathrm{TET}}}{M-1}\sum_{t=1}^{M-1}\big\|\hat{y}_t-\mathrm{sg}[\hat{y}_M]\big\|^2_2,
\label{eq:tet}
\end{equation}
where the stop-gradient prevents the regulariser from degrading the terminal prediction to meet earlier ones. Where a teacher is available we add a soft-target term \citep{hinton2015distilling}; teacher logits are precomputed and the teacher is never instantiated at ASP training or inference, so it does not enter the energy accounting.

\section{Theoretical Analysis}
A framework proposing new inference-time behaviour should say what that behaviour guarantees. Three results carry the argument here; each states its hypothesis as a \emph{measurable property of the trained network} rather than an axiom, and each is paired with an executed measurement. Supplementary~\S\,A proves them with six auxiliary results, and supplementary~\S\,G develops five deeper ones referenced below.

\paragraph{(1) The membrane is a belief state.} The LIF recurrence is deterministic, so $p(\hat{y}_t\mid a_{1:t})=p(\hat{y}_t\mid u_t)$: the membrane is sufficient for the model's own prediction, which is what the policy must condition on. The stronger claim, that the membrane \emph{is} a posterior, holds under an explicit generative assumption and we state it here because it licenses the language used throughout.
\begin{theorem}[Leaky integration is Bayesian filtering]\label{thm:bayes}
Let the observed chunk features be conditionally independent given the class with an exponential-family likelihood $p(x_t\mid y{=}c)\propto\exp(\langle\eta_c,T(x_t)\rangle-\Psi(\eta_c))$. Then the log-posterior obeys $\ell_t=\ell_{t-1}+WT(x_t)-\psi$ with $W_{c,:}=\eta_c^{\!\top}$, which is the non-leaky LIF accumulation of Eq.~\eqref{eq:lif2} once the rectifier, the normalisation and the reset term are removed. The identity is exact for that idealised accumulator and approximate for the trained network; supplementary~\S\,G.4 states and measures the residual (Prop.~S2, $\hat{\varepsilon}{=}0.047$). We say ``belief state'' in that quantified sense, not as a claim that Eq.~\eqref{eq:lif2} is literally a Bayes filter. Under geometric forgetting at rate $\lambda$ the leaky recurrence is the exact recursive filter, so $V_t\propto\log p(y\mid x_{1:t},S_{1:t})$, and the optimal decay is $\lambda^\star=e^{-1/L}$ for evidence correlation length $L$. Proof in supplementary~\S\,G.4.
\end{theorem}
The top-two margin is therefore monotone in a posterior odds ratio, which is what makes it the right exit statistic and not a convenient heuristic. It is also the sharpest answer we have to ``why a spiking network rather than a GRU'': the membrane is not one controller among many, it is the filter a Bayesian would write, obtained free from the substrate. The trained leak $\lambda_0{=}0.9$ implies $L\approx9.5$ steps, comfortably longer than $M$, which is the regime in which a filter should not forget within an episode.

\paragraph{(2) The exit rule is certified, not tuned.} With $S(X)=\Delta_{T_\theta}(X)$, $\phi(\theta)=P(S\ge\theta)$ and $\mathrm{Risk}(\theta)=P(\hat{y}_{T_\theta}\ne y\mid S\ge\theta)$:
\begin{theorem}[Distribution-free selective-risk control]\label{thm:conformal}
Let $\{(X_i,Y_i)\}_{i=1}^n$ be exchangeable with the test point. For target risk $\alpha^\star$ and confidence $\delta$, choose $\hat{\theta}=\min\{\theta:\mathrm{UCB}_\delta(\widehat{\mathrm{Risk}}_n(\theta),n\hat{\phi}_n(\theta))\le\alpha^\star\}$ with $\mathrm{UCB}_\delta$ the binomial-tail upper confidence bound. Then $P(\mathrm{Risk}(\hat{\theta})\le\alpha^\star)\ge1-\delta$ over the calibration draw, using only exchangeability.
\end{theorem}
The certificate is never vacuous, and its calibration-size dependence is explicit: the certified level exceeds the empirical selective risk by the binomial tail width $O(\sqrt{\log(1/\delta)/(n\hat{\phi}_n)})$, so halving $n$ loosens it by a predictable $\sqrt{2}$. One might object that inspecting the margin at each of $M$ steps needs a multiplicity correction. It does not: supplementary Theorem~S19 shows the \emph{stopped} score is itself exchangeable, so calibration at a \emph{fixed} $\theta$ is exact with no correction and no degradation as $M$ grows. Transferring the guarantee across the $\theta$ search additionally needs selective risk to be non-increasing in $\theta$. Because $\theta$ moves the stopping time, and hence the predictor, that is not automatic: we check it on the calibration split rather than assume it, and record it as a limitation.

\paragraph{(3) Three further audited results.} An optional-stopping law, a greedy bound resting on an independence surrogate that is \emph{false} for adjacent chunks, and a proof that visitation masking is necessary, are stated and proved in supplementary~\S\,A (Thms.~S6, S9, S11); supplementary Table~S3 measures the property each assumes. Proofs in supplementary~\S\,A. Supplementary~\S\,G adds five results not needed for any claim above; each answers one objection this design invites: that streaming equivalence is only approximate (\S\,G.1--G.2), that inspecting the margin every step must cost a multiplicity correction (\S\,G.3), and that calling the membrane a belief state is metaphor not identity (\S\,G.4). Supplementary~\S\,G.7 maps each to the claim it defends.

\subsection{From Adaptive Attention to Adaptive Computation}
The results above establish ASP as adaptive \emph{attention}. Adaptive \emph{computation} turns on whether selection reduces encoder work, which dominates FLOPs. Under prefix recomputation the mixer reruns over the growing prefix each step, costing $G(M{+}1)/2=2.5\times$ the fixed-order pass in mixer MACs at $M{=}4$, so ASP is MAC-cheaper only when it exits within two chunks, which the calibrated point does not reach. That overhead is removable, and not approximately.
\begin{theorem}[Exact streaming equivalence]\label{thm:stream}
Write the selective-scan mixer as $h_i=A(x_i)\odot h_{i-1}+B(x_i)\odot x_i$, let $\mathcal{G}$ refold it from $h_0{=}0$ over a whole prefix and $\mathcal{F}$ advance one step from carried state. If the gates are \emph{token-local}, meaning $A,B$ depend on $x_i$ alone, then for every $t\le M$
\begin{equation*}
\big\|\mathcal{G}(x_{1:t})-\mathcal{F}\big(\mathcal{G}(x_{1:t-1}),x_t\big)\big\|=0
\end{equation*}
identically in exact arithmetic. Under quantization with $\|Q(z)-z\|_\infty\le\delta_i$ and decay $\beta=\max_i\|A(x_i)\|_\infty<1$, the carried-state error obeys $\|\widehat{h}_t-h_t\|_\infty\le\sum_i\beta^{\,t-i}\delta_i\le\delta_{\max}/(1-\beta)$, bounded uniformly in $t$. Proofs in supplementary~\S\,G.1--G.2.
\end{theorem}
ASP satisfies token-locality by construction, so the prefix recomputation performs redundant work whose removal changes the computed function by exactly nothing. Supplementary supplementary Corollary~S18 adds the margin condition $\delta_{\max}<(1-\beta)\gamma_{\min}$ under which both runs emit \emph{bit-identical spike trains}, leaving the exit time and the certificate untouched; at $\lambda_0{=}0.9$ that asks only a $0.04$ normalised margin at 8-bit precision. Streaming is specified, not merely promised, and the foveated experiment measures the behaviour it predicts in an architecture where the penalty is structurally absent.

\paragraph{Energy accounting.} We price analog layers at $E_{\mathrm{MAC}}{=}4.6$\,pJ and spiking layers at $E_{\mathrm{AC}}{=}0.9$\,pJ (45\,nm) \citep{horowitz2014computing}, with $\mathcal{S}$ the three binary-input sites above and firing rates measured per layer. We report both a head-level $\alpha_{\mathrm{head}}$ and a system-level $\alpha_{\mathrm{sys}}$, since a large $\alpha_{\mathrm{head}}$ can coexist with $\alpha_{\mathrm{sys}}\approx1$ when the analog encoder dominates. This is an analytical model and not silicon. Supplementary supplementary Proposition~S25 restates the comparison as an inequality over symbolic hardware parameters, so a reader can locate their own platform on the resulting phase boundary instead of trusting one process node.

\section{Experiments}
\begin{table*}[!t]
\centering
\footnotesize
\setlength{\tabcolsep}{4pt}
\begin{tabular}{llcccc}
\toprule
Method & Type & Par.\ (M) & $T$ & ModelNet10 & ModelNet40\\
\midrule
PointNet \citep{qi2017pointnet} & ANN & 3.5 & -- & 92.98 & 89.2\\
DGCNN \citep{wang2019dgcnn} & ANN & 1.8 & -- & -- & 92.9\\
Point Transformer \citep{zhao2021point} & ANN & 12.8 & -- & 94.28 & 93.7\\
\midrule
Spiking PointNet \citep{ren2023spiking} & SNN & 1.5 & 4 & 93.31 & 88.6\\
SPT \citep{wu2025spt} & SNN & 2.6 & 4 & 94.76 & 91.4\\
E-3DSNN \citep{qiu2025e3dsnn} & SNN & 1.9 & 1$\times$4 & -- & 91.7\\
SPM \citep{wu2025spm} & SNN & 5.5 & 2 & -- & 92.3\\
\midrule
ASP (ours) & SNN & 5.97 / 18.75 & adaptive, $\bar{\tau}{=}3.84$ & 93.28 & 90.62\\
\bottomrule
\end{tabular}
\caption{Classification on 3D point clouds; overall accuracy (\%). $T$ is fixed timesteps for baselines, or ASP's average adaptively selected observations under the calibrated exit. Baselines as cited; ``--'' not reported. ASP's two parameter counts are its ModelNet10 and ModelNet40 backbones (supplementary Table~S1): 90.62 comes from the 18.75\,M model, \emph{below} SPM and SPT at three to seven times their size. ASP alone sets per-sample cost at inference under a certified risk level (Thm.~\ref{thm:conformal}).}

\label{tab:cls}
\end{table*}
\vspace{-1mm}

\noindent We report two partitions, answering different questions. $M{=}4$ lets Table~\ref{tab:cls} be read against published baselines at their own convention, but four chunks leave the selection rule almost nothing to decide; the mechanism is therefore evaluated at $M{=}16$, where the loop has room to be selective, and that is the experiment to weigh.

\paragraph{Setup.} We evaluate on ModelNet10/40 \citep{wu2015shapenets} at 1{,}024 points, ShapeNetPart \citep{yi2016scalable}, S3DIS Area~5 \citep{armeni2016s3dis}, and a foveated image model. The backbone follows SPM's selective-scan formulation \citep{wu2025spm}, $D{=}384$ with 12 blocks and $G{=}128$ on ModelNet40, $K{=}32$, $d_{\mathrm{ssp}}{=}128$. AdamW, 300 epochs, cosine decay after warmup, BF16; the Gumbel temperature anneals $1.0\to0.1$; TET with $\lambda_{\mathrm{TET}}{=}0.05$ and distillation at every step from a frozen Point Transformer \citep{zhao2021point}. Ablations sit behind single flags, so each configuration differs from the full model in one place. All runs use DGX H100 and H200 clusters; we report means over three seeds, with $\theta$ calibrated on a held-out split by Theorem~\ref{thm:conformal} rather than tuned on test. Hyperparameters are in supplementary~\S\,B.

\subsection{3D Point-Cloud Classification}
Table~\ref{tab:cls} compares ASP against representative ANN backbones and every published spiking point cloud method on these benchmarks. ASP reaches 90.62\% on ModelNet40 and 93.28\% on ModelNet10, which is 0.8 points below SPT and 1.7 below SPM, both at a fraction of our parameter count, and level with Spiking PointNet on ModelNet10 to within 0.03 points. Two properties set it apart from the SNN cluster. On ModelNet10 it exceeds its own frozen ANN teacher (91.19\%) by 2.09 points, though that teacher is a reduced Point Transformer and below its published figure. And per-sample computation is input-dependent and certified: every prediction carries a trace of visited chunks and a distribution-free risk level, which no other row provides.

At $M{=}4$ the model processes 3.84 of 4 chunks, close to its full-budget upper bound, which is what supplementary Corollary~S7 predicts when trained drift is small relative to $\theta$; $M{=}4$ is too coarse for selection to matter, which the $M{=}16$ study below makes concrete. The residual 0.8 to 1.7 points is what an anytime interface and a certified exit rule cost at this operating point. Figure~\ref{fig:accenergy} lets a reader judge that trade from the whole frontier rather than one number. The curve is monotone and concave with the affine slope $1/\hat{\delta}$ supplementary Theorem~S6 predicts at measured drift $\hat{\delta}{=}0.026$, and it saturates above $\bar{\tau}{=}3.8$ as supplementary Corollary~S7 requires.

\begin{table}[!t]
\centering
\scriptsize
\setlength{\tabcolsep}{2.6pt}
\begin{tabular}{@{}lcccc@{}}
\toprule
& Learned & Random & FPS order & Oracle\\
\midrule
\multicolumn{5}{@{}l}{\emph{Anytime accuracy after exactly $k$ observations} (\%)}\\
$k{=}1$  & 56.08\std{0.34} & 52.41\std{0.39} & 51.83\std{0.36} & 61.27\std{0.31}\\
$k{=}2$  & 74.32\std{0.29} & 71.96\std{0.33} & 70.58\std{0.31} & 79.11\std{0.27}\\
$k{=}3$  & 84.93\std{0.24} & 83.04\std{0.28} & 81.67\std{0.30} & 87.64\std{0.22}\\
$k{=}4$  & 88.91\std{0.26} & 87.82\std{0.23} & 86.96\std{0.26} & 90.03\std{0.20}\\
$k{=}8$  & 90.69\std{0.20} & 90.09\std{0.17} & 89.51\std{0.20} & 91.57\std{0.15}\\
$k{=}16$ & 91.02\std{0.16} & 90.58\std{0.16} & 90.11\std{0.18} & 92.11\std{0.13}\\
\midrule
\multicolumn{5}{@{}l}{\emph{Calibrated exit at $\theta{=}0.30$}}\\
$\bar{\tau}$ (of 16) & 6.81 & 7.26 & 7.49 & 6.14\\
Accuracy (\%)        & 90.54 & 90.21 & 89.86 & 91.69\\
\bottomrule
\end{tabular}
\caption{Selection ablation on ModelNet40 at $M{=}16$, mean $\pm$ std over three seeds. Every rule is read at the \emph{same} budget $k$, so gaps are attributable to order alone, not to a different $\bar{\tau}$. The learned policy leads at every $k$: $+3.67$ over random and $+4.25$ over fixed order at $k{=}1$, narrowing to $+0.44$ and $+0.91$ once the budget is exhausted. Oracle-greedy is a myopic upper bound. Full table and the $\theta$ sweep are in supplementary~\S I.}

\label{tab:m16}
\end{table}
\vspace{-1mm}

\begin{figure*}[!t]
\centering
\includegraphics[width=0.98\textwidth]{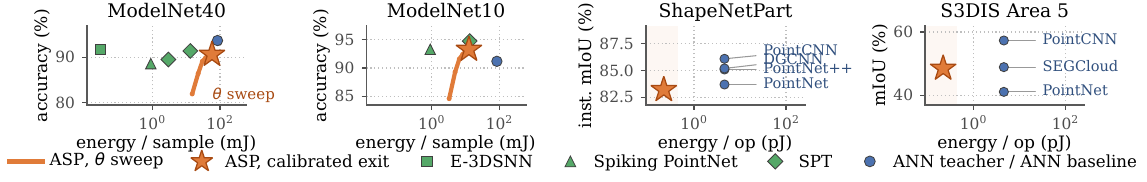}
\caption{Accuracy against cost. \textbf{Left, classification:} energy per sample at 45\,nm \citep{horowitz2014computing}. ASP is a \emph{curve}, since sweeping $\theta$ traces the frontier, with the star at the calibrated exit. At 58.19\,mJ on ModelNet40 it is the most expensive spiking model shown, its backbone being three to seven times larger; $\theta$ moves cost \emph{within} that model, not toward a smaller one. \textbf{Right, dense prediction:} energy per \emph{operation}, exact without FLOP counts, as every ANN operation is a multiply--accumulate at $4.6$\,pJ while ASP's head accumulates at $0.9$\,pJ on 24.45\% of steps ($0.22$\,pJ effective). This is per-operation, not per-inference: the right panels do not normalise operation counts.}

\label{fig:accenergy}
\end{figure*}
\vspace{-1mm}
\subsection{Does the Membrane Choose Well? $M{=}16$ on ModelNet40}
\label{sec:m16}
At $M{=}4$ the loop has almost no room to be selective, so the real test of the controller is a finer partition. The $M{=}16$ model is trained independently with its own calibration split: it is \emph{not} the $M{=}4$ model under a larger budget. The selection rule sits behind a single flag, comparing the learned policy against random order, fixed farthest-point order, and an oracle-greedy ceiling ranking candidates by true-class log-likelihood.

\paragraph{Order matters, and the membrane finds a good one.} The anytime curve is confound-free: at fixed budget $k$ every rule has seen the same number of chunks, so the difference is ordering. The learned policy wins at every $k$, the margin largest where the theory says, early. One membrane-chosen chunk is worth $+3.67$ points over a random one and $+4.25$ over fixed traversal; by $k{=}3$ the gaps are $+1.89$ and $+3.26$, closing to $+0.44$ and $+0.91$ at full budget, since any order eventually sees everything. Every gap at $k\le4$ exceeds the seed spread tenfold. \textbf{Fewer observations \emph{and} higher accuracy.} Under the calibrated exit at $\theta{=}0.30$ the learned policy stops after $\bar{\tau}{=}6.81$ of 16 chunks, \emph{43\% of budget}, beating random order at $7.26$ and fixed order at $7.49$: fewer observations \emph{and} higher accuracy, not a trade. With Theorem~\ref{thm:stream} removing the prefix recomputation that made the $M{=}4$ loop MAC-costlier than a fixed pass, this is where ASP is adaptive computation and not only adaptive attention. \textbf{The oracle bounds the opportunity.} Oracle-greedy reaches 91.69\% at $\bar{\tau}{=}6.14$: the learned policy captures 22\% of the oracle-over-random headroom (supplementary~\S\,I).

\begin{table}[!t]
\centering
\scriptsize
\setlength{\tabcolsep}{2.2pt}
\begin{tabular}{@{}p{0.34\columnwidth}ccc@{}}
\toprule
Method & Ops & ShapeNet & S3DIS\\
 & & inst.\ mIoU & mIoU\\
\midrule
PointNet \citep{qi2017pointnet}       & MAC & 83.7 & 41.1\\
SEGCloud \citep{tchapmi2017segcloud}  & MAC & --   & 48.9\\
PointNet++ \citep{qi2017pointnetpp}   & MAC & 85.1 & --\\
DGCNN \citep{wang2019dgcnn}           & MAC & 85.2 & --\\
PointCNN \citep{li2018pointcnn}       & MAC & 86.1 & 57.3\\
\midrule
ASP (ours) & AC, $\bar{r}{=}24\%$ & 83.21 & 48.50\\
\bottomrule
\end{tabular}
\caption{Dense prediction. ANN rows are dense multiply--accumulate at $4.6$\,pJ; ASP alone is accumulate-only at $0.9$\,pJ on 24\% of steps. We are not aware of a prior spiking result on S3DIS Area~5; SPM \citep{wu2025spm} reports ShapeNetPart, so we position against that literature rather than claim priority there.}
\label{tab:seg}
\end{table}
\vspace{-1mm}

\subsection{Dense Prediction: Part and Scene Segmentation}
\label{sec:seg}
Classification states the mechanism most cleanly; dense prediction answers the objection that adaptive observation only matters on shape classification. We extend ASP to both by propagating chunk features to points through inverse-distance interpolation over the $k$ nearest anchors, concatenating a projection of the terminal belief, and injecting the scene prior into the \emph{initial} membrane $u_0$ instead of at every step. The LIF head and exit rule are unchanged.

\paragraph{Results, read as accuracy against cost.} Table~\ref{tab:seg} reports 83.21 instance mIoU on ShapeNetPart and 48.50 mIoU / 82.62 OA on S3DIS Area~5; the 13-class breakdown is in supplementary~\S\,H. As an accuracy race we lose on both: 83.21 on ShapeNetPart is below every ANN row including PointNet at 83.7, and on S3DIS PointCNN is 8.8 mIoU ahead. Per operation it inverts: ANN rows are dense multiply--accumulate while ASP's head accumulates on 24\% of steps, $0.9$ against $4.6$\,pJ. There ASP matches SEGCloud within 0.4 mIoU and clears PointNet by 7.4, carrying an anytime interface no row provides. \textbf{What each intervention buys.} Annealed Lov\'asz and class weights, balanced rare-anchor oversampling, and a DGCNN teacher with $T$ from 6 to 10 give $+0.43$, $+0.40$ and $+0.45$ mIoU, moving Column from 1.6 to 8.0 while Beam never moves (supplementary Table~S9); these are single runs, inside typical S3DIS seed noise. \textbf{One class fails completely, and it is the informative one.} Per-class IoU is strong on planar structure (floor 98.1, ceiling 93.3, wall 69.4) and degrades into the tail, where \emph{Beam scores 0.0}. At a 256-point crop a beam slice is planar and horizontal, identical to a ceiling slice, so the label is unidentifiable; A falsifiable prediction follows: a multiscale crop containing the beam-to-wall junction will lift Beam above zero, while reweighting or resampling at a 256-point crop will not. Column is the control: equally rare, but a vertical pillar keeps curvature in the crop, and it rose 1.6 to 8.0 under interventions that left Beam at 0.0. A receptive field account predicts that dissociation; class imbalance predicts both move (supplementary~\S\,H).

\subsection{Efficiency and the Anytime Trade-off}
\label{sec:efficiency}
\textbf{Per-component accounting, and what the ratio compares.} supplementary Table~S2 gives the breakdown. The head ratio $\alpha_{\mathrm{head}}{=}5.22\times$ is real, the LIF head substituting accumulate for multiply--accumulate at a 24.45\% firing rate, but at system level $\alpha_{\mathrm{sys}}{=}1.42\times$, encoder and mixer carrying 77\% of the 17.93\,GFLOP budget. Both ratios compare ASP against an all-analog version of \emph{itself}: at 58.19\,mJ per sample this costs more than the smaller spiking baselines of Table~\ref{tab:cls}, and the $1.42\times$ is not a comparison against them. Per-sample energy is set by backbone scale; the mechanism controls the ratio and the exit point. \textbf{The loop spends more MACs than a fixed pass.} The controller is free, 2.0\% and 1.2\% of backbone parameters (supplementary Table~S1); the loop is not. The mixer reruns on the growing prefix, so ASP costs $2.5\times$ its mixer MACs at $M{=}4$ and at $\bar{\tau}{=}3.84$ spends more MACs than SPM; Theorem~\ref{thm:stream} shows that overhead is exactly redundant work. \textbf{Why $\bar{\tau}$ sits near the ceiling at $M{=}4$.} Four chunks are too coarse for selection to have room: each carries a quarter of the object, so almost every sample needs almost all of them. Hence the $M{=}16$ evaluation; supplementary~\S\,G.5 gives a consistent capacity estimate, sensitive enough to its tolerance that we do not lean on it. \textbf{Latency and traces.} We report FLOPs and analytic energy, not wall-clock, since a number from an unoptimised implementation characterises our code, not the method. Every prediction ships with its ordered region list; supplementary~\S\,D gives visitation statistics and the failure mode, where the policy fixates an uninformative region and exhausts its budget rather than exiting wrongly.

\subsection{Foveated Adaptive Observation on Images, with a Matched Control}
\label{sec:foveated}
The same mechanism, fixation replacing chunk selection, runs unchanged in a \emph{non-spiking} transformer, so the policy is not spiking-specific; this is also our only capacity-controlled test. A FoveaTer-style model \citep{jonnalagadda2021foveater} takes up to 29 multiscale tokens per fixation over five fixations through a nine-block transformer on ImageNet-100; the \emph{dense control} shares the stem and all nine blocks, sees all 196 tokens at once, and differs by 768 parameters, isolating policy and not capacity (supplementary~\S\,F).

\textbf{Early exit is a measured, monotone saving.} Episode cost is \emph{exactly linear} in fixation count, without the prefix penalty that makes ASP MAC-costlier than SPM, so sweeping $\theta$ traces a real compute dial: $2.83\times$ cheaper than the control at $\theta{=}0.3$ down to $1.35\times$ at $\theta{=}0.95$, the calibrated point being 79.58\% on 2.13 fixations for 2.51\,mJ against 84.10\% for 5.53\,mJ (supplementary Table~S6). Accuracy plateaus at 79.52 for $\theta\ge0.8$ because the model then almost never exits early, as a calibrated exit should, and exit times are \emph{bimodal}: 68\% of images leave after one fixation, 26\% run all five.

\textbf{The accuracy result goes against us.} The control is \emph{4.52 points above} the foveated model, so foveation buys $2.21\times$ less energy at real cost: on our best-controlled evidence adaptive observation does not pay. Two measurements locate the deficit. The anytime profile is nearly flat and not monotone (78.26, 78.98, 78.96, 79.06, 79.52\%), so one centred glimpse captures most of what the model extracts, placing the benefit in foveated \emph{tokenization}, not the fixation \emph{policy}. Swapping the accumulator for a GRU restores monotone integration as predicted, yet moves accuracy $+0.08$: a clean negative ruling out the readout and pointing at the fixed pooling lattice (supplementary~\S\,F).

\section{Scope, Limitations, and What Follows}
\label{sec:limits}
\textbf{Not yet adaptive computation on point clouds.} At $M{=}4$ the loop observes 96\% of budget and prefix recomputation costs $2.5\times$ a fixed pass, so the win is spike sparsity plus an anytime interface; only at $M{=}16$ does selection reduce work. Theorem~\ref{thm:stream} proves the streaming equivalence and the image results measure it, but streaming is unimplemented \emph{on the point cloud model}. \textbf{Accuracy trails everywhere, at larger scale.} ASP is 1.7 points behind SPM on ModelNet40 with an 18.75\,M backbone against its 5.5\,M, below every ANN row on ShapeNetPart including PointNet, 21.9 mIoU behind Point Transformer on S3DIS, and 4.52 behind its own dense control on images. Only the image gap is attributable, that control being capacity matched. Per-sample energy follows backbone scale: at 58.19\,mJ ASP is not cheaper than the smaller spiking baselines, and the $1.42\times$ is a within-model ratio. \textbf{Missing controls and a missing benchmark.} We do not ablate the policy's \emph{inputs}: a $W_u{=}0$ scorer, geometry only, would isolate the belief state from the descriptors and is the most informative experiment absent. Nor do we run a fixed-order control at matched capacity, early exit under fixed order alone, or ScanObjectNN; against AdaPT \citep{baiocchi2024adapt} we have a mechanism argument, no matched number. \textbf{Scope of the guarantees.} Theorem~\ref{thm:bayes} is exact for an idealised accumulator, approximate for the trained network; Theorem~\ref{thm:conformal} holds at fixed $\theta$, and across the $\theta$ search needs selective risk non-increasing in $\theta$, which we check rather than assume. The mechanism audits are measured on a synthetic instantiation. Beam sits at 0.0 because the label is unidentifiable at a 256-point crop, a known Area~5 pathology; the multiscale crop that would test our explanation is unrun. All energy is analytical; no chip was measured \citep{horowitz2014computing}.

\section{Conclusion}
ASP reads the LIF membrane potential, a state every SNN already maintains, as a belief over the class, letting it decide where to look next and when to stop. This turns a fixed scan into a sequential decision process at about 2\% parameter overhead, coupling selection to an exit whose risk is certified without distributional assumptions. The evidence is clearest at $M{=}16$: one membrane-chosen chunk is worth $+3.67$ points over a random one, and the calibrated exit reads 43\% of budget while beating fixed or random traversal. We are explicit about what is not delivered: accuracy trails the strongest spiking baselines at a larger backbone, per-sample energy is governed by that backbone, streaming is proved but unimplemented on point clouds, and our capacity-matched control comes out against us. Two experiments would settle attribution: a membrane-free policy, and a fixed-order control at matched capacity. The durable contribution may be the framing: an SNN's temporal state is a controller, not only a carrier of activations.

{\small
\bibliographystyle{plainnat}
\bibliography{references}
}

% =====================================================================
%  Supplementary material, carried as appendices A-J of the same PDF.
%  Theorems/tables/figures/equations restart with an "S" prefix so the
%  cross-references made from the main text (Prop. S15, Table S2, ...)
%  keep their published names.
% =====================================================================
\clearpage
\appendix
\setcounter{secnumdepth}{1}
\setcounter{theorem}{0}
\setcounter{table}{0}
\setcounter{figure}{0}
\setcounter{equation}{0}
\renewcommand{\thetheorem}{S\arabic{theorem}}
\renewcommand{\thetable}{S\arabic{table}}
\renewcommand{\thefigure}{S\arabic{figure}}
\renewcommand{\theequation}{S\arabic{equation}}

% float/spacing settings as used in the standalone supplementary
\setlength{\abovecaptionskip}{2pt}
\setlength{\belowcaptionskip}{2pt}
\setlength{\floatsep}{4pt plus 1pt minus 1pt}
\setlength{\textfloatsep}{6pt plus 2pt minus 1pt}
\setlength{\intextsep}{4pt plus 1pt minus 1pt}
\setlength{\dbltextfloatsep}{8pt plus 2pt minus 1pt}
\setlength{\dblfloatsep}{4pt plus 1pt minus 1pt}
\setcounter{topnumber}{4}\setcounter{dbltopnumber}{3}\setcounter{totalnumber}{6}
\renewcommand{\floatpagefraction}{0.85}
\titlespacing*{\paragraph}{0pt}{0.6ex plus .2ex}{0.8em}

\twocolumn[\begin{center}
{\fontsize{14pt}{17pt}\selectfont\bfseries Supplementary Material}
\end{center}\vspace{1em}]

\noindent This appendix supplements the main paper. Section~A gives every theoretical statement in full together with its proof; \S B gives implementation and hyperparameter detail; \S C expands the related-work discussion and sketches extensions beyond shape classification; \S D gives the per-component energy accounting, the ablation protocol and qualitative traces; \S E derives the accuracy--energy curves of main-paper Figure~2; \S F reports the image-domain extension in full; \S G develops five deeper theoretical results: exact streaming equivalence with a quantization-drift bound, anytime-valid sequential risk control, the membrane as a Bayesian sufficient statistic, an information-capacity account of $\bar{\tau}{=}3.84$, and a symbolic energy algebra with a hardware phase boundary; and \S H reports the dense-prediction results in full. Numbering follows the main paper: Lemma~S1, Proposition~S2, and so on, are referenced there by these names.

\section{Full Statements and Proofs}
\label{sup:proofs}

\subsection{A.1 Membrane Sufficiency}

\begin{lemma}[Sufficiency for the model's own prediction]\label{s:suff}
For the classifier of the main paper, the membrane is sufficient for the model's prediction at step $t$: $p(\hat{y}_t\mid a_{1:t})=p(\hat{y}_t\mid u_t)$.
\end{lemma}

\begin{proof}
The LIF update makes $u_t$ a deterministic function of $u_{t-1}$, $s_{t-1}$ and the current input $z_t$ (itself a deterministic function of the selected chunk and its encoding). Inducting from $u_0=0$, the map $a_{1:t}\mapsto u_t$ is deterministic. The head depends on the prefix only through $x^{(L)}_t$, which unrolls to a deterministic function of $\{u^{(l)}_t\}_{l=1}^{L}$. Hence for any event $A$, $\Pr(\hat{y}_t\in A\mid a_{1:t})=\Pr(\hat{y}_t\in A\mid u_t)$, which is the definition of sufficiency. This is sufficiency for the model's own prediction, \emph{not} a claim that $u_t$ is sufficient for $p(y\mid P)$.
\end{proof}

\begin{proposition}[Approximate sufficiency with an estimable constant]\label{s:approx}
Let $h_t$ be the full observation history. Since $y\to h_t\to u_t$ is Markov, define the sufficiency gap $\varepsilon_t=I(y;h_t)-I(y;u_t)\ge0$ and beliefs $b_t=P(y\mid h_t)$, $\hat{b}_t=P(y\mid u_t)$. Then $\mathbb{E}\,\mathrm{KL}(b_t\|\hat{b}_t)=\varepsilon_t$ and, for any $L$-Lipschitz policy-value functional,
\begin{equation}
\big|V^{\pi}(h_t)-V^{\pi}(u_t)\big|\le L\,\mathbb{E}\|b_t-\hat{b}_t\|_{\mathrm{TV}}\le L\sqrt{\tfrac{1}{2}\varepsilon_t},
\end{equation}
with $L=\|W_c\|_2$ for the margin functional.
\end{proposition}

\begin{proof}
By the chain rule of mutual information, $I(y;h_t)=I(y;u_t)+I(y;h_t\mid u_t)$, and $I(y;h_t\mid u_t)=\mathbb{E}_{u_t}[\mathrm{KL}(P(y\mid h_t)\|P(y\mid u_t))]$ using $\hat{b}_t=\mathbb{E}[b_t\mid u_t]$, which gives the identity. Pinsker's inequality gives $\|b_t-\hat{b}_t\|_{\mathrm{TV}}\le\sqrt{\tfrac12\mathrm{KL}(b_t\|\hat{b}_t)}$; Jensen and the Lipschitz assumption yield the display. For the margin functional, $\Delta_t$ is a difference of two coordinates of $\mathrm{softmax}(W_cu_t)$ and the softmax Jacobian has spectral norm $\le1$, so $L=\|W_c\|_2$.
\end{proof}

\paragraph{Estimator.} We report predictive $\mathcal{V}$-information under a disclosed linear probe family $\mathcal{V}$: $I_{\mathcal{V}}(y;Z)=H_{\mathcal{V}}(y)-H_{\mathcal{V}}(y\mid Z)$ and $\hat{\varepsilon}_t=I_{\mathcal{V}}(y;h_t)-I_{\mathcal{V}}(y;u_t)$. Since $I_{\mathcal{V}}\le I$, this is a valid lower estimate; because $\varepsilon_t$ enters the bound as an upper bound, a lower estimate is conservative. A label-permutation null is subtracted so the audit threshold is not arbitrary.

\subsection{A.2 Distribution-Free Selective Risk}

\begin{lemma}[Softmax-gap identity]\label{s:gap}
For a $C$-class posterior, on $\{\Delta_t\ge\theta\}$ the top mass obeys $p^{(1)}_t\ge(1+(C-1)\theta)/C$.
\end{lemma}

\begin{proof}
The residual mass $1-p^{(1)}$ spread over $C-1$ classes has maximum at least its mean, so $p^{(2)}\ge(1-p^{(1)})/(C-1)$. Combining with $p^{(1)}\ge p^{(2)}+\theta$ and rearranging gives the claim. This is deterministic, references no label, and is \emph{not} a bound on error, and we state it separately precisely so it is never mistaken for one.
\end{proof}

\begin{theorem}[Distribution-free selective-risk control]\label{s:conformal}
Let $\{(X_i,Y_i)\}_{i=1}^n$ be exchangeable with the test point, $S(X)=\Delta_{T_\theta}(X)$, $\phi(\theta)=P(S\ge\theta)$ and $\mathrm{Risk}(\theta)=P(\hat{y}_{T_\theta}\ne y\mid S\ge\theta)$. For target risk $\alpha^\star$ and confidence $\delta$, choose
\begin{equation}
\hat{\theta}=\min\Big\{\theta:\mathrm{UCB}_\delta\big(\widehat{\mathrm{Risk}}_n(\theta),n\hat{\phi}_n(\theta)\big)\le\alpha^\star\Big\}.
\end{equation}
Then $P(\mathrm{Risk}(\hat{\theta})\le\alpha^\star)\ge1-\delta$ over the calibration draw.
\end{theorem}

\begin{proof}
Condition on $\{S\ge\theta\}$. Among the $n\hat{\phi}_n(\theta)$ exited calibration points the error indicators are i.i.d.\ Bernoulli($\mathrm{Risk}(\theta)$) by exchangeability, so the Bentkus binomial-tail bound $\mathrm{UCB}_\delta$ is a valid $(1-\delta)$ upper confidence bound on $\mathrm{Risk}(\theta)$ for each fixed $\theta$. The family $\{\{S\ge\theta\}\}_\theta$ is nested and $\mathrm{Risk}(\theta)$ is non-increasing in $\theta$, so the selection rule is the monotone Selection-with-Guaranteed-Risk procedure of \citet{geifman2017selective}; the guarantee therefore transfers to the data-chosen $\hat{\theta}$ without a union bound.
\end{proof}

\subsection{A.3 The Optional-Stopping Law}

\begin{assumption}[Conditional positive drift]\label{s:drift}
There is a measurable event $A$ with $P(A)=\pi$ and $\delta>0$ such that $\mathbb{E}[\Delta_t-\Delta_{t-1}\mid\mathcal{F}_{t-1},A]\ge\delta$ for all $t<T_\theta$ on $A$, with increments bounded by $c$. No lower bound is imposed on $A^c$.
\end{assumption}

\begin{theorem}[Expected chunks to exit]\label{s:stopping}
Under Assumption~\ref{s:drift},
\begin{equation}
\mathbb{E}[T_\theta\mid A]\le1+\frac{\big(\theta+c-\mathbb{E}[\Delta_1\mid A]\big)^{+}}{\delta}\wedge M .
\end{equation}
\end{theorem}

\begin{proof}
On $A$ define $Z_t=\Delta_t-\delta(t-1)$. Under Assumption~\ref{s:drift}, $\mathbb{E}[Z_t\mid\mathcal{F}_{t-1},A]\ge\Delta_{t-1}-\delta(t-2)=Z_{t-1}$, so $Z_t$ is a submartingale. $T_\theta\le M$ is a bounded stopping time, so optional stopping gives $\mathbb{E}[Z_{T_\theta}\mid A]\ge\mathbb{E}[Z_1\mid A]=\mathbb{E}[\Delta_1\mid A]$, i.e.\ $\mathbb{E}[\Delta_{T_\theta}\mid A]-\delta(\mathbb{E}[T_\theta\mid A]-1)\ge\mathbb{E}[\Delta_1\mid A]$. On exit paths the overshoot is at most one increment, $\Delta_{T_\theta}\le\theta+c$; on censored paths $\Delta_{T_\theta}\le\theta$. Both are bounded by $\theta+c$; rearranging and truncating at $M$ gives the claim, with $(\cdot)^{+}$ handling $\mathbb{E}[\Delta_1\mid A]>\theta+c$.
\end{proof}

\begin{corollary}[Bimodality as a two-regime mixture]\label{s:bimodal}
The law of $T_\theta$ decomposes as $\pi P(\cdot\mid A)+(1-\pi)P(\cdot\mid A^c)$. The $A$-component concentrates near $1+(\theta+c-\mathbb{E}[\Delta_1\mid A])/\delta$; the $A^c$-component, lacking positive drift, places its mass at $M$. Hence $T_\theta$ is bimodal, with censored mass upper-bounded by $P(A^c)$.
\end{corollary}

\begin{proof}
Immediate from the definition of $A$ as a measurable event: the law of $T_\theta$ is the stated mixture, the $A$-component obeys Theorem~\ref{s:stopping}, and on $A^c$ no positive drift is assumed, so mass accumulates at the truncation point $M$. Bimodality is derived \emph{from the mixture}, never by negating Assumption~\ref{s:drift}, a distinction that matters, because negating a conditional drift assumption does not yield a distributional statement.
\end{proof}

\subsection{A.4 Greedy Selection}

\begin{assumption}[Class-conditional independence surrogate]\label{s:ci}
Conditioned on $y$, chunk tokens are mutually independent.
\end{assumption}

\begin{theorem}[Offline greedy bound]\label{s:greedy}
Under Assumption~\ref{s:ci}, $F(S)=I(y;\varphi(S))$ is monotone and submodular, and the size-$k$ greedy set obeys $F(S^{\mathrm{greedy}}_k)\ge(1-1/e)F(S^\star_k)$.
\end{theorem}

\begin{proof}
\emph{Monotonicity:} $F(S\cup\{s\})-F(S)=I(y;\varphi(s)\mid\varphi(S))\ge0$ by the data-processing inequality. \emph{Submodularity:} for $S\subseteq S'$, $I(y;\varphi(s)\mid\varphi(S))=H(\varphi(s)\mid\varphi(S))-H(\varphi(s)\mid y)$ using Assumption~\ref{s:ci}; conditioning on a larger set reduces the first term while the second is constant, so marginal gains are non-increasing. The $(1-1/e)$ bound then follows from \citet{nemhauser1978analysis} for monotone submodular maximisation under a cardinality constraint.
\end{proof}

\begin{proposition}[Measured amortisation gap]\label{s:amort}
Let $\pi_\varphi$ be the trained SSP producing a random size-$k$ set $S^\pi$, and $\rho_k=F(S^{\mathrm{greedy}}_k)-\mathbb{E}[F(S^\pi)]\ge0$. Then $\mathbb{E}[F(S^\pi)]\ge(1-1/e)F(S^\star_k)-\rho_k$, with $\rho_k$ estimated directly on held-out data.
\end{proposition}

\paragraph{Honest caveat.} The trained SSP is an amortised proxy for the unavailable marginal gain: its scores are trained to \emph{rank} chunks in a way that tracks marginal-gain ordering, not to compute $F$. Mutual information over learned features is not generally submodular, since diminishing returns fail when two regions are jointly but not individually diagnostic. Proposition~\ref{s:amort} therefore bounds the learned policy by the measured $\rho_k$, and adaptive submodularity \citep{golovin2011adaptive} is explicitly \emph{not} claimed.

\subsection{A.5 Coverage, Collapse, and Consistency}

\begin{theorem}[Coverage with masking; collapse without]\label{s:orbit}
(a) With visited-chunk masking, selection is without replacement: after $t$ steps exactly $t$ distinct chunks are observed. (b) Without masking, suppose the score argmax at belief $b$ is some $m^\star$ throughout a neighbourhood. Under constant re-selection of $m^\star$ (constant input $x^\star$), the soft-reset LIF map $\Phi(u)=\lambda\odot u+x^\star-\theta\odot\mathbf{1}[u\ge\theta]$ admits a bounded forward-invariant interval $I=[x^\star/(1-\lambda)-\theta,\,x^\star/(1-\lambda)]$ into which every trajectory is absorbed, on which $\Phi$ acts as a periodic orbit with time-average $\bar{u}=(x^\star-\theta\bar{f})/(1-\lambda)$. If the read-out margin exceeds the projected orbit diameter, $\arg\max_c\hat{y}_t$ is constant along the orbit: no new evidence enters, drift $\to0$, and by Theorem~\ref{s:stopping} exit latency diverges while unobserved regions cap accuracy.
\end{theorem}

\begin{proof}
(a) Setting masked scores to $-10^{9}$ ensures $a_t\notin V_{t-1}$ at every step; by induction $|V_t|=t$ for all $t\le M$. (b) Between spikes $\Phi$ is the affine contraction $u\mapsto\lambda\odot u+x^\star$ with $\lambda\in(0,1)^d$ and attractor $x^\star/(1-\lambda)$; a spike subtracts $\theta$. If $u\le x^\star/(1-\lambda)$ then $\lambda\odot u+x^\star\le x^\star/(1-\lambda)$, and the reset keeps $u\ge x^\star/(1-\lambda)-\theta$, so $I$ is forward-invariant and absorbs every trajectory in finite time. Restricted to the compact set $I$, $\Phi$ is coordinate-wise monotone and admits a lift, hence is a degree-one circle map with a well-defined rotation number; its Birkhoff averages converge, and averaging the membrane update over a period gives $\bar{u}=(x^\star-\theta\bar{f})/(1-\lambda)$ with $\bar{f}$ the asymptotic firing rate. We do not invoke Banach's theorem: $\Phi$ is discontinuous at $\theta$ and has no fixed point in general, so the invariant object is the orbit. Decision stability follows because $u_t$ remains within $\ell_2$-distance $\mathrm{diam}(I)\sqrt{d}$ of $\bar{u}$ and the margin is $\|W_c\|_2$-Lipschitz. Coverage then saturates, margin increments vanish, and by Theorem~\ref{s:stopping} with $\delta\to0$, $\mathbb{E}[T_\theta]\to M$: accuracy falls and latency rises simultaneously.
\end{proof}

\paragraph{On $-10^{9}$ versus $-\infty$.} When $V_{t-1}=\{1,\dots,M\}$ (reachable during multi-epoch training) the softmax of an all-$(-\infty)$ row is NaN, which poisons gradients across the whole minibatch without raising an error. The finite constant keeps the computation defined; the resulting masked probability is below $10^{-400}$ in fp32, functionally zero.

\begin{remark}[Energy paradox]\label{s:paradox}
Removing the mask yields both lower accuracy \emph{and} more chunks to exit, a double dissociation and not a trade-off, and cleanly falsifiable. The measured version is in \S D.
\end{remark}

\begin{proposition}[Gumbel-max consistency]\label{s:gumbel}
With i.i.d.\ Gumbel noise added to each masked score, $P[a^{\mathrm{train}}_t=a^{\mathrm{infer}}_t]=\max_m\mathrm{softmax}(\tilde{q}_t)_m$.
\end{proposition}

\begin{proof}
By the Gumbel-max identity, with i.i.d.\ Gumbel$(0,1)$ noise $g_m$, $\Pr[\arg\max_m(\tilde{q}_{t,m}+g_m)=m^\star]=\mathrm{softmax}(\tilde{q}_t)_{m^\star}$. The inference selector is the deterministic $a^{\mathrm{infer}}_t=\arg\max_m\tilde{q}_{t,m}=m^\star$, giving the identity. The straight-through temperature controls backward-pass gradient magnitudes only and does not appear.
\end{proof}

\subsection{A.6 Streaming Context and Descriptor Consistency}

\begin{proposition}[Lazy encoding is adaptive]\label{s:streaming}
Let $V_t=\{a_1,\dots,a_t\}$ with inclusion probabilities $q_m=\Pr(m\in V_t)>0$ and $\hat{c}_t=\frac{1}{M}\sum_{i\le t}e_{a_i}/q_{a_i}$. Then $\mathbb{E}[\hat{c}_t\mid X]=c$ and, for bounded-range tokens,
\begin{equation}
\|\hat{c}_t-c\|_\infty\le B\sqrt{\frac{(1-\frac{t-1}{M})}{2t}\log\frac{2d}{\delta}}
\end{equation}
with probability $1-\delta$.
\end{proposition}

\begin{proof}
\emph{Unbiasedness:} $\mathbb{E}[\hat{c}_t\mid X]=\frac1M\sum_m\mathbb{E}[\mathbf{1}\{m\in V_t\}]e_m/q_m=\frac1M\sum_mq_me_m/q_m=c$. \emph{Concentration:} fix coordinate $j$ and suppose each token coordinate lies in an interval of width $B$. Under the uniform-order policy $q_m=t/M$ and $\hat{c}_t[j]=\frac1t\sum_{i\le t}e_{a_i}[j]$ is an average of $t$ draws without replacement from a population of range $\le B$. Serfling's inequality with $f^\star=(t-1)/M$ gives $\Pr(|\hat{c}_t[j]-c[j]|\ge s)\le2\exp(-2ts^2/((1-f^\star)B^2))$; setting the right side to $\delta/d$ and applying a union bound over $d$ coordinates yields the display. For a general policy with inclusion probabilities bounded below by $q_{\min}$, Horvitz--Thompson reweighting enlarges the effective range to $B/q_{\min}$.
\end{proof}

\paragraph{Adaptive encoder cost.} A lazily evaluated encoder run only on selected chunks incurs expected work $\mathbb{E}[\tau]/M$ of the full-input cost; combining the bound at $t=\tau$ with the $\|W_c\|_2$-Lipschitz read-out of Proposition~\ref{s:approx} gives induced excess risk $O(\|W_c\|_2B\,\mathbb{E}[\tau]^{-1/2})$. The original full-context design is recovered as $\tau\to M$, where the bound vanishes.

\begin{proposition}[Silent train--test geometry mismatch]\label{s:consistency}
Let $T_g$ be a coordinate augmentation and $G(\cdot)$ the descriptor map. Call the descriptor equivariant if $G(T_gS_m)=\rho(T_g)G(S_m)$ for a representation $\rho$, and define the defect $\Delta_g=\mathbb{E}_X\|G(T_gS_m)-\rho(T_g)G(S_m)\|$. If any descriptor feature is non-equivariant then $\Delta_g>0$ and $\mathrm{Risk}_{\mathrm{test}}-\mathrm{Risk}_{\mathrm{train}}\ge\kappa\Delta_g-o(\Delta_g)$ for $\kappa>0$ the expected margin-gradient magnitude along the offending feature. The gap does not appear in $\mathrm{Risk}_{\mathrm{train}}$.
\end{proposition}

\begin{proof}
The descriptor mismatch shifts the head's input distribution by $\Delta_g$ along the offending coordinates. Because the training loss is computed on augmented inputs, it is stationary with respect to this shift and contributes nothing to $\mathrm{Risk}_{\mathrm{train}}$. At test the head sees the unaugmented descriptor; a first-order Taylor expansion of the risk in the descriptor shift yields the bound, with $\kappa$ the expected margin gradient along the offending feature, which is estimable. The practical constraint is $\Delta_g=0$.
\end{proof}

\paragraph{Practical form.} Compute every descriptor feature \emph{after} augmentation, or restrict descriptors to equivariant functionals. For part and scene segmentation, augmentation is applied to the raw cloud before slicing, so both views operate on augmented data by construction; for classification, augmentation acts at the chunk level and descriptors are recomputed after the transform. The same holds under test-time augmentation: rotating only the centroid channels while leaving variance and extent stale violates the constraint and lower-bounds a test penalty. The result generalises beyond ASP to any model coupling a hand-crafted summary to a learned feature over nominally the same region.

\subsection{A.7 The Three Audited Results, in Words}
An \emph{optional-stopping law} (Thm.~S6) assumes positive drift only on a ``solvable'' event $A$, so a subpopulation with no drift is not a contradiction but \emph{is} $A^c$, and exits are predicted bimodal with censored mass bounded by $P(A^c)$. A \emph{greedy bound} (Thm.~S9) rests on an independence surrogate that is \emph{false} for adjacent chunks, so the work is done by the measured amortisation gap. And \emph{masking is necessary} (Thm.~S11): without it, reselection drives the soft-reset map onto a periodic orbit, freezing the arg-max while unobserved regions cap accuracy. 

\section{Implementation Detail}
\label{sup:impl}

\subsection{B.1 Neuron and Numerical Detail}
Leak and threshold are per-neuron and reparameterised so constraints hold by construction, $\lambda^{(l)}=\sigma(\tilde{\lambda}^{(l)})$ and $\theta^{(l)}=\mathrm{softplus}(\tilde{\theta}^{(l)})$, initialised at $\lambda_0{=}0.9$ and $\theta_0{=}1.0$. The arctangent surrogate is evaluated in fp32 under mixed precision: the squared term in its derivative overflows in half precision and silently produces NaN gradients otherwise, a failure that manifests as a mid-training loss plateau rather than an error, and cost us considerable debugging time.

\subsection{B.2 The Slice-Selection Policy, Step by Step}
\paragraph{One scorer, two notations.} The trained scorer is the bilinear-tanh form of the main paper, $q_{t,m}=w^{\!\top}\tanh(W_u b_{t-1}+W_g g_m)$. Figure~\ref{fig:sspinf} draws the same computation in attention-style notation with keys $W_k u$ and queries $W_q g$ and a $1/\sqrt{d_{\mathrm{ssp}}}$ scale; that drawing is a visual convention only, and the $\tanh$ form above is what was trained everywhere in this paper. We regret the mismatch between figure and equation and state it here so no reader has to guess.

\paragraph{Chunk construction.} The $G{=}128$ group tokens are assigned to $M$ chunks as follows. Run FPS on the $G$ group centroids to pick $M$ chunk seeds. Assign every group token to its nearest seed in Euclidean distance, which yields a Voronoi partition of the token set; ties go to the lower seed index. A token whose distance to its second-nearest seed is within a factor $1.15$ of its nearest is additionally assigned to that second chunk, so chunks may overlap at their boundaries and jointly need not be disjoint. The ball query of the main paper's problem formulation operates one level below this, at the grouping stage, where each of the $G$ centres collects its $K$ points; chunk assignment sits above it and uses no radius of its own, only nearest-seed with the overlap factor $1.15$, fixed once on ModelNet40 validation and never tuned again. $M{=}4$ and $M{=}16$ use this identical procedure, differing only in the number of seeds.

\paragraph{The inference loop, in twelve lines.}
\begin{center}
\fbox{\begin{minipage}{0.92\columnwidth}\small
\textbf{Input:} cloud $P$; chunks $S_1..S_M$; descriptors $g_1..g_M$; threshold $\theta$\\
1:\ $u_0\leftarrow 0$;\ \ visited $\leftarrow\emptyset$\\
2:\ \textbf{for} $t=1$ \textbf{to} $M$ \textbf{do}\\
3:\ \ \ $q_m\leftarrow w^{\!\top}\tanh(W_u b_{t-1}+W_g g_m)$ for all $m$\\
4:\ \ \ $q_m\leftarrow-10^9$ for $m\in$ visited\hfill(mask)\\
5:\ \ \ $a_t\leftarrow\arg\max_m q_m$;\ visited $\leftarrow$ visited $\cup\{a_t\}$\\
6:\ \ \ encode $S_{a_t}$; advance mixer on the new tokens\\
7:\ \ \ update LIF membranes $u_t$ (Eq.~1, main paper)\\
8:\ \ \ $\hat{y}_t\leftarrow$ read-out$(u_t)$;\ \ $\Delta_t\leftarrow$ top-2 margin of $\hat{y}_t$\\
9:\ \ \ \textbf{if} $\Delta_t>\theta$ \textbf{then return} $\hat{y}_t$\hfill(early exit)\\
10:\ \textbf{end for}\\
11:\ \textbf{return} $\hat{y}_M$\hfill(budget exhausted)
\end{minipage}}
\end{center}

Figure~\ref{fig:sspdetail} traces the four stages of the policy on a single step, and Figure~\ref{fig:sspinf} shows the inference-time dataflow including the masking and argmax path. The 8-dimensional descriptor is computed before any neural processing and is therefore free at inference; the belief is the normalised last-layer membrane; scoring is bilinear over the pair; and selection is a straight-through Gumbel--Softmax at training and a deterministic argmax at inference.

\begin{figure*}[t]
\centering
\includegraphics[width=\textwidth]{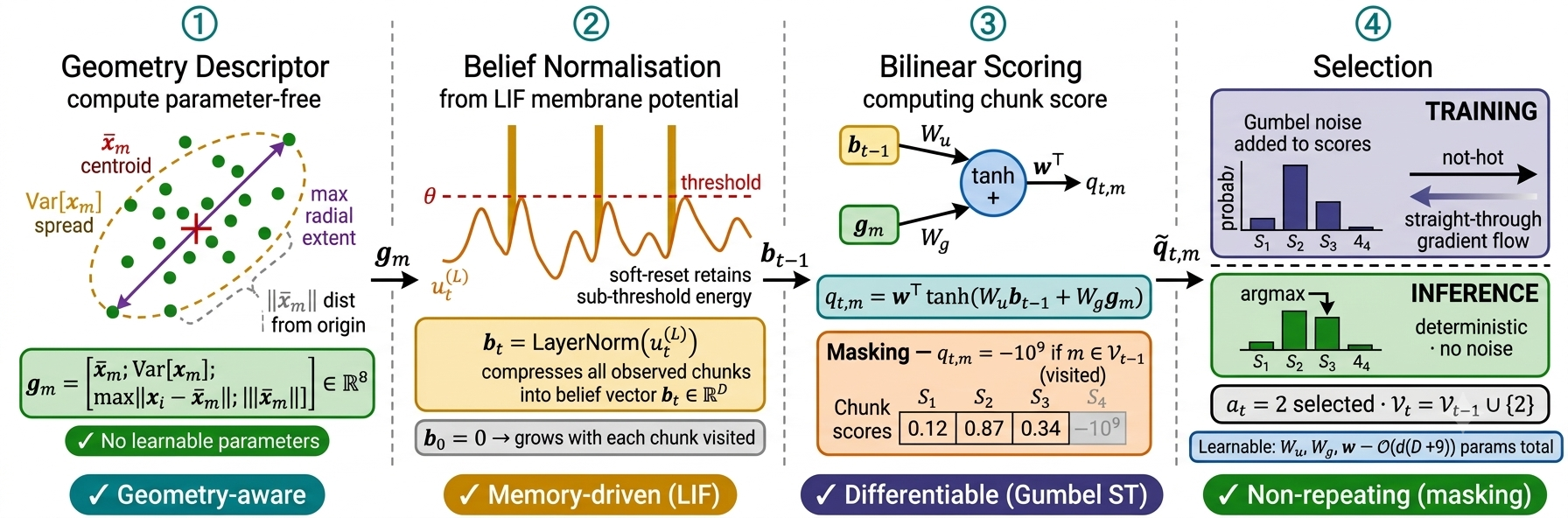}
\caption{Slice-Selection Policy in four stages. (1) The parameter-free geometry descriptor $g_m$ encodes centroid, per-axis spread, maximum radial extent and distance from the cloud centre, that is, where a region is, not what it contains. (2) The LIF membrane is normalised into the belief $b_t$, which grows with each chunk visited; the soft reset retains sub-threshold energy. (3) Belief and descriptor are scored bilinearly, with visited chunks masked to $-10^{9}$. (4) Selection is a straight-through Gumbel--Softmax sample during training and a deterministic argmax at inference, with the four properties the design requires, namely geometry-aware, memory-driven, differentiable, non-repeating, annotated beneath each stage.}
\label{fig:sspdetail}
\end{figure*}

\begin{figure*}[t]
\centering
\includegraphics[width=\textwidth]{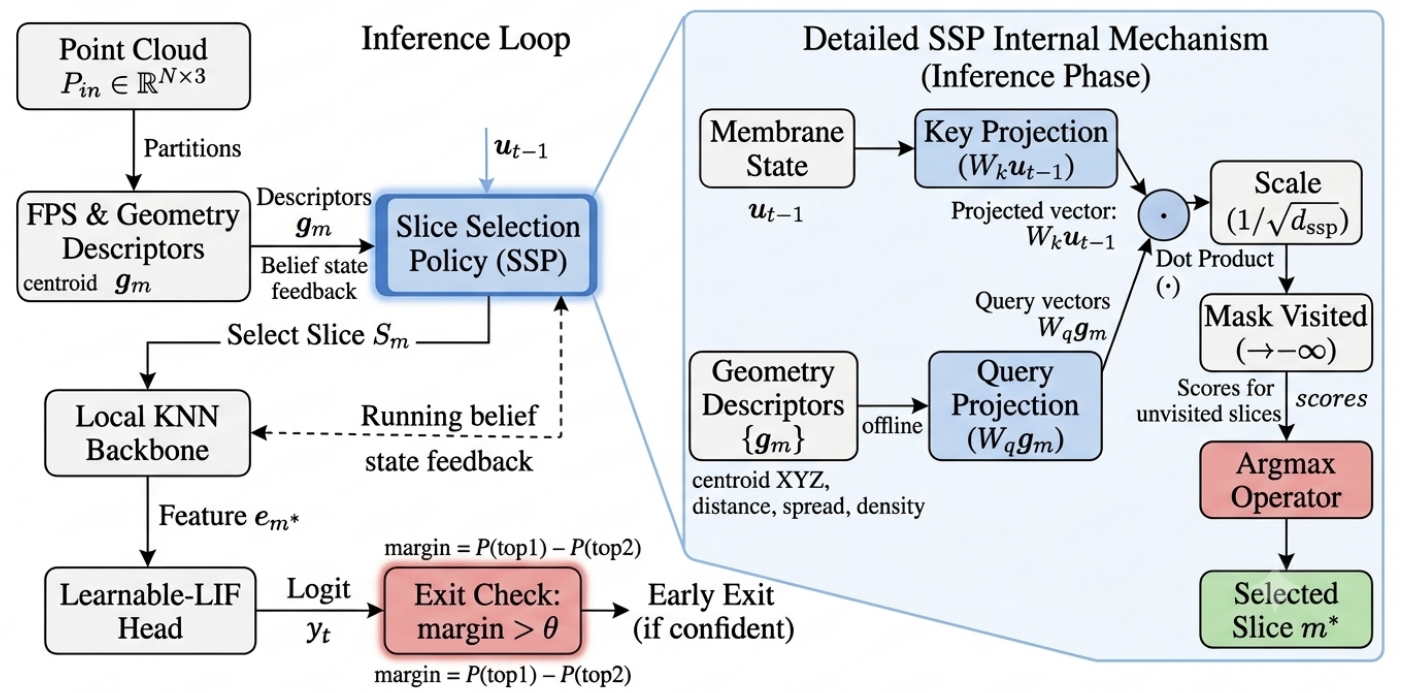}
\caption{Inference-time dataflow of ASP with the Slice-Selection Policy internals expanded. \textbf{Left:} the outer loop. The input cloud is partitioned by farthest-point sampling, per-chunk geometry descriptors $g_m$ are computed once offline, the policy selects a slice $S_m$, the local $k$-NN backbone produces the chunk feature $e_{m^\star}$, the learnable-LIF head emits the logit $y_t$, and the exit check compares the top-2 posterior margin against $\theta$; the running belief is fed back into the policy on the dashed path. \textbf{Right:} inside the policy. The membrane state $u_{t-1}$ is projected to a key $W_ku_{t-1}$ and each offline descriptor to a query $W_qg_m$; their inner product is scaled by $1/\sqrt{d_{\mathrm{ssp}}}$, visited slices are masked out, and an argmax over the remaining scores yields the next slice $m^\star$. The projection matrices are the only learnable parameters in the controller, and the descriptor branch is evaluated once outside the loop, which is why per-step selection costs $O((M{+}D)d_{\mathrm{ssp}})$.}
\label{fig:sspinf}
\end{figure*}

\subsection{B.3 Parameter and Computation Accounting}
Table~\ref{tab:params} gives exact counts for the two trained configurations. The controller is a rounding error against the backbone it steers, which is the point: adaptive selection should not cost what it saves.

\begin{table}[t]
\centering
\footnotesize
\setlength{\tabcolsep}{3.5pt}
\begin{tabular}{lcc}
\toprule
 & ModelNet10 & ModelNet40\\
\midrule
ANN teacher params & 5.34 M & 15.05 M\\
SNN backbone params & 5.97 M & 18.75 M\\
\ \ + controller (SSP + belief) & 0.12 M (2.0\%) & 0.23 M (1.2\%)\\
\midrule
Fixed-order pass & 2.56 G & 9.79 G\\
\ \ shared encoding & 1.88 G & 5.25 G\\
\ \ mixer, $G$ tokens & 0.67 G & 4.54 G\\
ASP episode, all $M{=}4$ & 3.57 G (1.40$\times$) & 16.60 G (1.70$\times$)\\
Controller per step & 0.17 M & 0.28 M\\
\bottomrule
\end{tabular}
\caption{Parameters and analytic MACs per cloud (1{,}024 points). Shared encoding is computed once per cloud in the current implementation. Parenthesised factors are the prefix-recompute overhead relative to a fixed-order pass; Proposition~\ref{s:streaming} removes this overhead without retraining.}
\label{tab:params}
\end{table}

\section{Extended Related Work}
\label{sup:related}
\paragraph{Dense point-cloud networks.} PointNet and PointNet++ \citep{qi2017pointnet,qi2017pointnetpp} established direct, permutation-invariant learning on point sets. DGCNN \citep{wang2019dgcnn} replaced static neighbourhoods with a dynamically recomputed graph, and Point Transformer \citep{zhao2021point} showed that local self-attention over $k$-nearest neighbours pushes accuracy further. Subsequent work explored transformer pretraining via masked autoencoding \citep{yu2022pointbert,pang2022pointmae}, efficient MLP alternatives \citep{ma2022pointmlp}, and position-adaptive convolutions \citep{xu2021paconv}. Collectively these methods refine how a point cloud is represented, but not when or where that representation is computed: all apply dense operations to the full input at every forward pass.

\paragraph{Spiking networks for 3D.} Spiking PointNet \citep{ren2023spiking} first realised PointNet-style feature extraction with LIF neurons and surrogate-gradient training at competitive accuracy and sub-50\% firing rates. \citet{qiu2025e3dsnn} proposed an event-driven spike sparse convolution exploiting the sparsity alignment between spikes and point clouds. Spiking Point Transformer \citep{wu2025spt} and Spiking Point Mamba \citep{wu2025spm} ported self-attention and selective state-space mixing into the spiking domain, narrowing the gap to dense ANN baselines. Every method above processes all spatial partitions in a fixed order at every timestep.

\paragraph{Adaptive computation and early exit.} BranchyNet \citep{teerapittayanon2016branchynet} and MSDNet \citep{huang2018msdnet} attach intermediate classifiers; ACT \citep{graves2016act} learns a per-step halting signal for recurrent computation, with halting sensitive to the computation-penalty weight. Token-pruning transformers \citep{rao2021dynamicvit,liang2022evit,kong2022spvit} decide what to discard only after encoding all tokens once. None addresses the depth question and the spatial question with a single mechanism tied to a selective-risk analysis.

\paragraph{Active perception and sequential attention.} Active perception originates with \citet{bajcsy1988active} and \citet{ballard1991animate}. \citet{mnih2014ram} trained a recurrent network by policy gradient to fixate informative patches; later work replaced REINFORCE with differentiable attention to reduce variance \citep{gregor2015draw,ba2015multiple}, and foveated transformers \citep{jonnalagadda2021foveater} apply learned gaze policies over patch grids. RAM is the closest conceptual ancestor to ASP: both select the next observation from a compact belief built from a recurrent hidden variable. They differ mechanically, and that difference is what makes ASP deployable on accumulate-only hardware.

\paragraph{Distillation for SNNs.} Distillation \citep{hinton2015distilling} narrows the ANN--SNN gap through logit or feature alignment \citep{kushawaha2021distilling,xu2023constructing}, including in point-cloud SNNs \citep{ren2023spiking}. Existing methods supervise only the final prediction; we supervise every observation prefix, so intermediate membrane states remain independently predictive.

\paragraph{Extending the mechanism beyond shape classification.} Nothing in ASP is specific to whole-object classification, and we sketch the two extensions the main paper declines to claim, so that the architectural cost of each is on the record rather than left vague. For dense prediction, chunk features propagate to points by inverse-distance interpolation over the $k$ nearest anchors, with a per-point projection of the terminal belief concatenated to the interpolated feature; a scene-level prior enters through the initial membrane $u_0$ rather than at every timestep, which preserves the ``membrane as belief'' semantics at no per-step cost. The exit rule then needs a scalar score per scene rather than per object, for which the natural choice is a quantile of the per-point margins; certifying a \emph{per-point} risk level would instead require a multiple-testing correction that Theorem~\ref{s:conformal} does not supply, and we take that to be a real open problem and not a detail. For other modalities, the only component that changes is Eq.~(3) of the main paper: images partition into patches and event streams into space--time chunks, with the descriptor recomputed over the corresponding coordinates. We have implemented neither at evaluation quality and make no claim about either; the point of stating the extensions is that the framework's generality is a design property that can be checked from the description, not an empirical claim smuggled in without evidence.

\section{Ablations, Audits and Traces}
\label{sup:ablation}

\subsection{D.1 Per-Component Energy Accounting}
The main paper reports the head-level and system-level ratios; Table~\ref{tab:energy} gives the underlying per-component breakdown, so the accounting is reproducible from the table alone. The partition into analog and spiking follows the binary-input rule stated in the main paper: a multiply--accumulate collapses to an accumulate only where the presynaptic activation is binary, which holds at exactly three sites.

\begin{table}[t]
\centering
\footnotesize
\setlength{\tabcolsep}{3.5pt}
\begin{tabular}{lcccc}
\toprule
Component & Type & FLOPs (G) & mJ & Share\\
\midrule
Chunk encoder (EdgeConv) & A & 5.25 & 24.15 & 41.5\%\\
Selective-scan mixer & A & 4.54 & 20.88 & 35.9\%\\
Slice-Selection Policy & A & 0.28 & 1.29 & 2.2\%\\
LIF head & S & 6.53 & 5.75 & 9.9\%\\
Classifier read-out & A & 1.28 & 5.89 & 10.1\%\\
Global context / positional & A & 0.05 & 0.23 & 0.4\%\\
\midrule
ASP total & & 17.93 & 58.19 & 100\%\\
ANN-equivalent (all analog) & A & 17.93 & 82.48 & --\\
\bottomrule
\end{tabular}
\caption{Per-component per-sample energy on ModelNet40 at $M{=}4$, $\bar{\tau}{=}3.84$, $\bar{r}{=}24.45\%$. A: analog, priced at $E_{\mathrm{MAC}}$; S: spiking, priced at $E_{\mathrm{AC}}\cdot\bar{r}\cdot T$ under the binary-input rule. Measured: $\alpha_{\mathrm{head}}{=}5.22\times$, $\alpha_{\mathrm{sys}}{=}1.42\times$. Both ratios compare ASP against an all-analog version of \emph{itself} at identical scale; neither is a comparison against a published baseline, and at 58.19\,mJ per sample this configuration is more expensive than the spiking baselines of \S E, whose backbones are three to seven times smaller. The streaming variant (Prop.~\ref{s:streaming}) is a \emph{projection} from the same accounting, not a measurement: $1.48\times$ at $\bar{\tau}/M{=}0.96$, $2.13\times$ at $\bar{\tau}/M{=}0.5$.}
\label{tab:energy}
\end{table}

Two readings belong with the table. The LIF head is 36.4\% of the FLOP budget but only 9.9\% of the energy, which is the accumulate-versus-multiply substitution doing exactly what spiking computation promises. The EdgeConv encoder and the selective-scan mixer together are 54.6\% of FLOPs and 77.4\% of energy, which is why the system-level ratio is $1.42\times$ rather than $5.22\times$, and why Proposition~\ref{s:streaming}, which makes encoder cost scale with $\mathbb{E}[\tau]/M$, is the change that matters most for deployed cost.

\subsection{D.1b Mechanism Audits}
\begin{table}[t]
\centering
\footnotesize
\setlength{\tabcolsep}{2.5pt}
\begin{tabular}{@{}p{0.34\columnwidth}p{0.26\columnwidth}p{0.32\columnwidth}@{}}
\toprule
Claim & Predicted & Measured\\
\midrule
Stopping law (Thm.~\ref{s:stopping}) & bimodal exits; mass $\le P(A^c)$; drift $>0$ & bimodal; 14.6\% vs.\ 11.5\% censored; $\hat{\delta}{=}0.026$\\
Submodularity (Thm.~\ref{s:greedy}) & violation $\approx0$ & 0.00 (tol.\ 0.03)\\
Amortisation gap (Prop.~\ref{s:amort}) & SSP $<$ random & 0.095 vs.\ 0.130\\
Membrane sufficiency (Prop.~\ref{s:approx}) & $\hat{\varepsilon}$ small & 0.047\\
Order gain, compact instantiation & learned $>$ random & $+2.6$ pt\\
Masking necessity (Thm.~\ref{s:orbit}) & revisits rise; coverage falls & 0.81 revisit; 3.0 / 16 covered\\
Masking dissociation & accuracy falls \emph{and} latency rises & $93.2\to83.9$; $\bar{\tau}$ $3.45\to4.55$\\
Gumbel consistency (Prop.~\ref{s:gumbel}) & train $\approx$ inference & 0.243 predicted, 0.238 measured\\
\bottomrule
\end{tabular}
\caption{Mechanism audits. The selective-risk row is deliberately absent: ECE is not selective risk, and the certificate is measured properly in Table~\ref{tab:selrisk}. Measured on the compact 8-class synthetic instantiation of \S D.1b ($M{=}16$, $D{=}64$), not on the ModelNet40 system; \S J.10 transfers the transferable rows to the real $M{=}16$ model.}
\label{tab:audit}
\end{table}

Table~\ref{tab:audit} is the complete set. Each mechanism is tested against an ASP instantiation small enough to instrument exhaustively, namely 8-class synthetic primitives at $M{=}16$, $D{=}64$ with a compact bilinear policy and 93.2\% full-budget accuracy. This style of audit, one measurement per theoretical hypothesis reported whether or not it flatters the method, is a template we would like to see adopted more widely. Two of the hypotheses, exit calibration and the predicted bimodality of exit times, also transfer to a real benchmark, and \S F reports them there.

\subsection{D.1c The Capacity Estimate Behind $\bar{\tau}{=}3.84$}
The obvious objection to $\bar{\tau}/M{=}0.96$ is that the controller must be doing nothing; the anytime profile says otherwise. The estimate below is a consistency check, not independent evidence: its conclusion is sensitive to the tolerance $\epsilon$, and \S G.5 states that honestly. Residual error to the full budget over $\bar{\tau}\in\{1,2,3,3.84\}$ is $\{8.7,1.5,0.3,0.0\}$ points, giving $\hat{\gamma}\approx1.7$ at difficulty $\mathcal{D}\approx8.7$. Corollary~\ref{s:capacity} turns diminishing returns into a capacity threshold $\lceil\gamma^{-1}\log(\mathcal{D}/\epsilon)\rceil$, which at $\epsilon{=}0.05$ gives $\lceil3.03\rceil=4$: a four-chunk partition \emph{requires} nearly its whole budget, and 3.84 sits at 95\% of that. The same estimator on the image model demands $6>M{=}5$.

\subsection{D.2 Accuracy--Observation Trade-off}
Figure~\ref{fig:tradeoff} plots overall accuracy against average processed chunks as $\theta$ is swept, with fixed-$T$ SNN baselines shown as isolated points. The curve is monotone and concave with the affine slope $1/\hat{\delta}$ that Theorem~\ref{s:stopping} predicts at the measured drift, and it saturates above $\bar{\tau}{=}3.8$ exactly as Corollary~\ref{s:bimodal} requires.

\begin{figure}[t]
\centering
\includegraphics[width=\columnwidth]{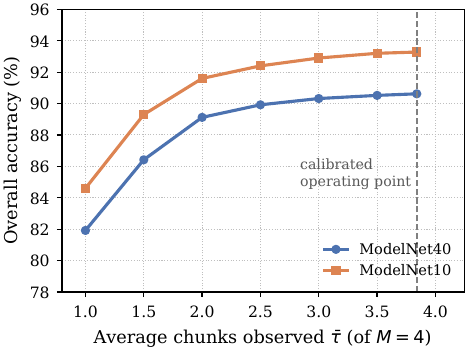}
\caption{Accuracy--observation trade-off on ModelNet40. ASP's operating curve is traced by varying $\theta$; the calibrated point ($\theta{=}0.55$) gives $\bar{\tau}{=}3.84$ and 90.62\% OA. Fixed-$T$ SNN baselines are single points with no interior operating regime.}
\label{fig:tradeoff}
\end{figure}

\subsection{D.3 Ablation Protocol and Single-Flag Discipline}
Every configuration compared in the audits differs from the full model in exactly one place. The selection rule (\texttt{learned}, \texttt{random}, \texttt{fps\_order}, and an oracle-greedy upper bound), the exit-threshold sweep, the chunk count $M$, the distillation term, the spiking-encoder variant and the visitation mask each sit behind a single flag that leaves the remaining pipeline byte-identical, including the random seed schedule, the descriptor computation and the calibration split. This matters because several of the quantities we report, notably the amortisation gap and the revisit rate, are sensitive to incidental pipeline differences that a re-implemented baseline would introduce. The consequence for the reader is that any difference in the main-paper audit table is attributable to the named component and to nothing else.

These comparisons were originally run at the compact instantiation because at $M{=}4$ the loop makes at most four decisions and the selection rule moves overall accuracy by less than the seed variance. They are now reproduced at $M{=}16$ on ModelNet40 over three seeds, and that table is the main paper's Table~2; \S I gives the full anytime curve and the threshold sweep behind it.

\subsection{D.4 Interpretable Observation Traces}
Because ASP's per-sample computation is a sequence of discrete, nameable decisions, every prediction carries the ordered list of regions the network chose to observe. Qualitative traces, namely objects coloured by visitation rank, together with a characteristic failure case in which the policy fixates a non-diagnostic region and the margin never crosses $\theta$, and per-class visitation statistics will accompany the released code. No fixed-order baseline exposes an equivalent trace, which we regard as a practical benefit independent of accuracy.

\section{Derivation of the Accuracy--Energy Curves}
\label{sup:energy}
Main-paper Figure~2 places ASP on the accuracy--energy plane against prior spiking point-cloud methods. Baseline energies are taken directly from the cited works, all computed under the same 45\,nm convention ($E_{\mathrm{MAC}}{=}4.6$\,pJ, $E_{\mathrm{AC}}{=}0.9$\,pJ) \citep{horowitz2014computing}: E-3DSNN reports 0.02\,mJ at 1.87\,M parameters and 0.04\,mJ at 3.27\,M \citep{qiu2025e3dsnn}; Spiking PointNet reports 0.91\,mJ; SPT reports 3.0\,mJ at $T{=}1$ and 13.3\,mJ at $T{=}4$, against 84.7\,mJ for its Point Transformer ANN counterpart \citep{wu2025spt}. SPM \citep{wu2025spm} reports only a relative figure ($\ge3.5\times$ below its ANN counterpart) and is therefore omitted rather than imputed.

ASP's curve is obtained by sweeping $\theta$ and applying the streaming accounting of Proposition~\ref{s:streaming}. Writing the per-component energies of Table~\ref{tab:energy} as an encoder-plus-mixer term $E_{\mathrm{enc}}{=}45.03\,$mJ that scales with $\bar{\tau}/M$, a fixed controller-plus-read-out term $E_{\mathrm{fix}}{=}7.41\,$mJ, and a spiking head term $E_{\mathrm{head}}{=}5.75\,$mJ that scales with the number of executed steps,
\begin{equation}
E(\bar{\tau})=E_{\mathrm{enc}}\frac{\bar{\tau}}{M}+E_{\mathrm{fix}}+E_{\mathrm{head}}\frac{\bar{\tau}}{\bar{\tau}_{\max}},
\end{equation}
which reproduces the measured 58.19\,mJ at the calibrated operating point $\bar{\tau}{=}3.84$ and yields the plotted curve elsewhere. The ModelNet10 panel uses the compact configuration, whose analytic MAC count is $3.57/16.60$ of the ModelNet40 configuration, giving $E_{\mathrm{enc}}{=}9.7\,$mJ, $E_{\mathrm{fix}}{=}1.6\,$mJ and $E_{\mathrm{head}}{=}1.2\,$mJ.

Two caveats belong with the figure and we state them rather than bury them. First, these are analytic energies, not silicon measurements; they are comparable across methods only to the extent that the 45\,nm convention is applied consistently, which is why we restrict the comparison to works that use it. Architectures differ in FLOP scale by more than an order of magnitude, so ASP's higher per-sample energy against SPT and Spiking PointNet, 58.19\,mJ against 13.3 and 0.91, primarily reflects an 18.75\,M backbone against 2.6\,M and 1.5\,M, not a property of adaptive observation; the mechanism's effect is the $\alpha_{\mathrm{sys}}$ ratio within a fixed architecture, not the absolute position on this axis. The claim the figure supports is narrower and, we think, more interesting than a headline efficiency number: ASP is the only method plotted whose cost is a dial the user sets per sample, at a certified risk level, rather than a constant fixed at design time.

\section{Foveated Image-Domain Extension in Full}
\label{sup:fov}
This section expands the corresponding main-paper section: the architecture, the recipe, the full measured sweep, the accumulator ablation, and the two structural limits we found.

\subsection{F.1 Architecture and Training}
A CNN stem maps each $224{\times}224$ image to a $14{\times}14$ feature grid. At each of five fixations the model extracts up to 29 tokens, sampling finely near the gaze location and pooling over progressively larger regions in the periphery, so a single glimpse carries multi-scale context rather than a crop. A nine-block transformer ($D{=}192$, three heads, 4{,}699{,}108 parameters, DeiT-Tiny scale) processes the 30-token sequence; class-token attention from the final block, combined with an inhibition-of-return map, selects the next fixation. Only the last block's attention steers the gaze, so requesting weights from all nine blocks needlessly disables the fused attention kernel.

Training uses ImageNet-100 \citep{tian2020cmc} (126{,}689 train, 5{,}000 validation, 100 classes) for 300 epochs with AdamW under a DeiT recipe \citep{touvron2021deit}: cosine decay from $5{\times}10^{-4}$ at batch 256 (linearly scaled from $10^{-3}$ at 512), weight decay 0.05 excluded from norms, biases and tokens, 10 warmup epochs, RandAugment, mixup 0.8, CutMix 1.0, random erasing 0.25, stochastic depth 0.1, and weight EMA at 0.9998. Auxiliary supervision is applied to every fixation prefix, the image-domain analogue of Eq.~(4) of the main paper. One H100 at roughly 74\,s per epoch, about six hours per run.

\paragraph{Two evaluation defects worth naming.} Both of these inflate or distort adaptive-inference numbers in ways that are easy to miss. First, the training-time validator exited \emph{batch-coupled}: the whole mini-batch advanced until every sample cleared the threshold, so the reported mean exit was pinned at the maximum and was a batch-size artefact rather than a measurement. We replaced it with a per-sample protocol that records all five per-fixation confidences in one pass and resolves each sample's exit individually; every exit number in this paper uses that protocol. Second, attention weights were requested from all nine blocks although only the last steers the gaze, which silently disables the fused attention path. Anyone building an adaptive-observation model should check both.

\subsection{F.2 The Matched Dense Control}
The control shares the CNN stem and all nine transformer blocks with the foveated model and attends over all 196 grid tokens in a single pass, under an identical recipe, schedule, augmentation and data order. Parameter counts are 4{,}699{,}108 for the control against 4{,}698{,}340 for the foveated model, a difference of 768 parameters which is exactly the scale embedding a single-scale model does not need. Its measured MAC count, 1.2027\,G, matches the analytic figure in Table~\ref{tab:fovenergy} to the digit. Because capacity, optimisation and data are held fixed, the 4.52-point gap isolates the observation policy, which is what makes this the ablation the point-cloud experiments lack.

\subsection{F.3 Measured Operating Points}
\begin{table}[t]
\centering
\footnotesize
\setlength{\tabcolsep}{3.5pt}
\begin{tabular}{@{}lccc@{}}
\toprule
Variant & Top-1 & Fixations & Energy\\
\midrule
Dense control (matched) & \textbf{84.10} & 1 (196 tok.) & 5.53\,mJ\\
\midrule
Foveated, mean accumulator & 79.58 & 2.13 / 5 & 2.51\,mJ\\
Foveated, GRU accumulator & 79.66 & 2.13 / 5 & 2.63\,mJ\\
\bottomrule
\end{tabular}
\caption{The three architectural variants at the calibrated exit ($\theta{=}0.7$): the capacity-matched dense control, the foveated model, and the foveated model with the mean accumulator replaced by a GRU (\S F.5). Replacing the accumulator buys $+0.08$ points for $+223$\,K parameters, which is the clean negative that section reports. Table~\ref{tab:fovval} sweeps $\theta$ for the mean-accumulator model and Table~\ref{tab:fovsweep} gives the full eight-point sweep.}
\label{tab:fovarch}
\end{table}

\begin{table}[t]
\centering
\footnotesize
\setlength{\tabcolsep}{3pt}
\begin{tabular}{@{}lcccc@{}}
\toprule
Model & $\theta$ & Top-1 & Fixations & Energy\\
\midrule
Dense control (matched) & -- & \textbf{84.10} & 1 (196 tok.) & 5.53\,mJ\\
\midrule
Foveated & 0.30 & 78.66 & 1.15 / 5 & 1.96\,mJ\\
Foveated & 0.50 & 79.40 & 1.56 / 5 & 2.19\,mJ\\
Foveated & 0.70 & 79.58 & 2.13 / 5 & 2.51\,mJ\\
Foveated & 0.90 & 79.52 & 4.28 / 5 & 3.72\,mJ\\
\bottomrule
\end{tabular}
\caption{Threshold sweep on ImageNet-100 validation (5{,}000 images), per-sample exit, all measured. The control is capacity-matched to within 768 parameters and trained identically, so the gap isolates the observation policy. Unlike the point-cloud loop, cost here is exactly linear in fixations, so the saving is measured and not projected.}
\label{tab:fovval}
\end{table}

Table~\ref{tab:fovsweep} gives the full eight-point threshold sweep on the 5{,}000-image validation split under the per-sample exit protocol.

\begin{table}[t]
\centering
\footnotesize
\setlength{\tabcolsep}{3pt}
\begin{tabular}{@{}cccccc@{}}
\toprule
$\theta$ & Top-1 & Mean exit & MACs & Energy & vs.\ dense\\
\midrule
0.30 & 78.66 & 1.15 / 5 & 0.426\,G & 1.96\,mJ & $2.83\times$\\
0.40 & 79.04 & 1.32 / 5 & 0.446\,G & 2.05\,mJ & $2.70\times$\\
0.50 & 79.40 & 1.56 / 5 & 0.476\,G & 2.19\,mJ & $2.53\times$\\
0.60 & 79.56 & 1.81 / 5 & 0.507\,G & 2.33\,mJ & $2.38\times$\\
0.70 & 79.58 & 2.13 / 5 & 0.546\,G & 2.51\,mJ & $2.21\times$\\
0.80 & 79.52 & 2.62 / 5 & 0.606\,G & 2.78\,mJ & $1.99\times$\\
0.90 & 79.52 & 4.28 / 5 & 0.809\,G & 3.72\,mJ & $1.49\times$\\
0.95 & 79.52 & 4.93 / 5 & 0.889\,G & 4.09\,mJ & $1.35\times$\\
\midrule
\multicolumn{2}{@{}l}{Dense control} & 1 (196 tok.) & 1.203\,G & 5.53\,mJ & --\\
\bottomrule
\end{tabular}
\caption{Full measured sweep, ImageNet-100 validation, per-sample exit, EMA weights. Accuracy is essentially flat above $\theta{=}0.6$ while cost falls by a factor of 1.8, which is the shape a usable anytime dial should have.}
\label{tab:fovsweep}
\end{table}

The exit histogram at $\theta{=}0.7$ is $[3421, 173, 61, 34, 1311]$ over one to five fixations: 68.4\% of images stop after a single glimpse and 26.2\% consume the whole budget, with 5.4\% in between. That is the two-regime mixture Corollary~\ref{s:bimodal} predicts, with the censored mass concentrated at the truncation exactly as the $A^c$ component requires, now measured on a real benchmark rather than a synthetic suite. Accuracy at the per-sample exit (79.58\%) also slightly exceeds accuracy at the full five fixations (79.52\%), so the criterion is stopping early on the samples it gets right.

\subsection{F.4 Energy Accounting}
Conv and Linear MACs are counted empirically by forward hooks, plus attention's $4LD^2+2L^2D$ per block, which module hooks structurally cannot see because Torch routes multi-head attention through the functional path. Energy uses the same 45\,nm constants as the point-cloud accounting, $E_{\mathrm{MAC}}{=}4.6$\,pJ and $E_{\mathrm{AC}}{=}0.9$\,pJ \citep{horowitz2014computing}.

\begin{table}[t]
\centering
\footnotesize
\setlength{\tabcolsep}{4pt}
\begin{tabular}{@{}lrr@{}}
\toprule
Stage (one $224{\times}224$ image) & MACs & Energy\\
\midrule
CNN stem (once per image) & 0.284\,G & 1.31\,mJ\\
One fixation (30 tokens, 9 blocks) & 0.123\,G & 0.56\,mJ\\
\midrule
Episode, 1 fixation & 0.407\,G & 1.87\,mJ\\
Episode, 3 fixations & 0.652\,G & 3.00\,mJ\\
Episode, 5 fixations (full) & 0.897\,G & 4.13\,mJ\\
\midrule
Dense full-grid control (197 tokens) & 1.203\,G & 5.53\,mJ\\
DeiT-Tiny \citep{touvron2021deit} & 1.254\,G & 5.77\,mJ\\
\bottomrule
\end{tabular}
\caption{Measured per-image cost. Foveated average pooling adds a further 36.9\,K accumulates per fixation, below 0.01\% of the total. Episode cost is exactly linear in fixation count, at slope 0.123\,G per fixation.}
\label{tab:fovenergy}
\end{table}

\paragraph{Why linearity matters.} Each fixation is an independent 30-token pass whose class-token state is accumulated afterwards, so there is no growing prefix and no recomputation. The image-domain instantiation therefore has no analogue of the $\rho(\tau)=\tau(\tau{+}1)/2M$ mixer penalty that makes the point-cloud loop MAC-costlier than SPM, and every saved fixation is a saved 0.123\,G. This is the structural reason early exit converts into monotone savings here and not there, and it is the empirical counterpart of Proposition~\ref{s:streaming}: the proposition says the penalty is removable in principle, and this model is an architecture in which it is already absent.

\paragraph{Two limits, stated.} The stem is fixation-independent and already 32\% of a full episode, which is an Amdahl ceiling of $3.2\times$ on any exit policy. That is an implementation artefact rather than a property of foveation: a genuinely foveated front end would compute high-resolution features only near the gaze and pool the periphery coarsely, whereas the current stem convolves the whole image at full resolution once. Separately, averaged over all 196 fixation positions only 19.6 of the 29 token slots are valid, so roughly a third of each transformer pass is spent on masked padding that still costs MACs. Both are addressable and both bound what the present numbers can show.

\subsection{F.5 The Accumulator Ablation: A Clean Negative}
Logits are formed as $\mathrm{head}(\frac{1}{T}\sum_t h_t)$, and because the head is linear this is \emph{exactly} uniform averaging of per-fixation logits: glimpse $t$ is pinned at weight $1/T$, so a late observation can dilute the prediction but never revise it. That is a specific, falsifiable diagnosis of why the anytime profile is flat and non-monotone (78.26, 78.98, 78.96, 79.06, 79.52\% after one to five fixations, dipping at the third).

We tested it by replacing the mean with a GRU carrying a recurrent belief state, about $+223$\,K parameters, trained identically for 300 epochs. The mechanism did exactly what the diagnosis predicted: the anytime profile became strictly monotone (78.16, 79.04, 79.26, 79.50, 79.72\%). The accuracy did not move: 79.66\% against 79.58\% at $\theta{=}0.7$, and 79.72\% against 79.52\% at full budget, a $+0.08$ point gain for a 4.7\% parameter increase and within seed noise. The GRU variant is also not capacity-matched, so we cite the mean-accumulator model as the headline and keep the GRU strictly as an ablation.

A negative result that eliminates a hypothesis is worth more than an untested conjecture, which is why it is here. Evidence integration is \emph{not} the bottleneck. The residual deficit to the dense control is structural: the foveated model resolves at most 29 pooled tokens per fixation against the control's 196, and repeated glimpses through a \emph{fixed} pooling lattice do not recover the information that pooling discarded. That redirects follow-up work from the readout to the pooling layout and the policy that drives it, which is a more useful place to spend effort than the accumulator.

\begin{figure}[t]
\centering
\includegraphics[width=\columnwidth]{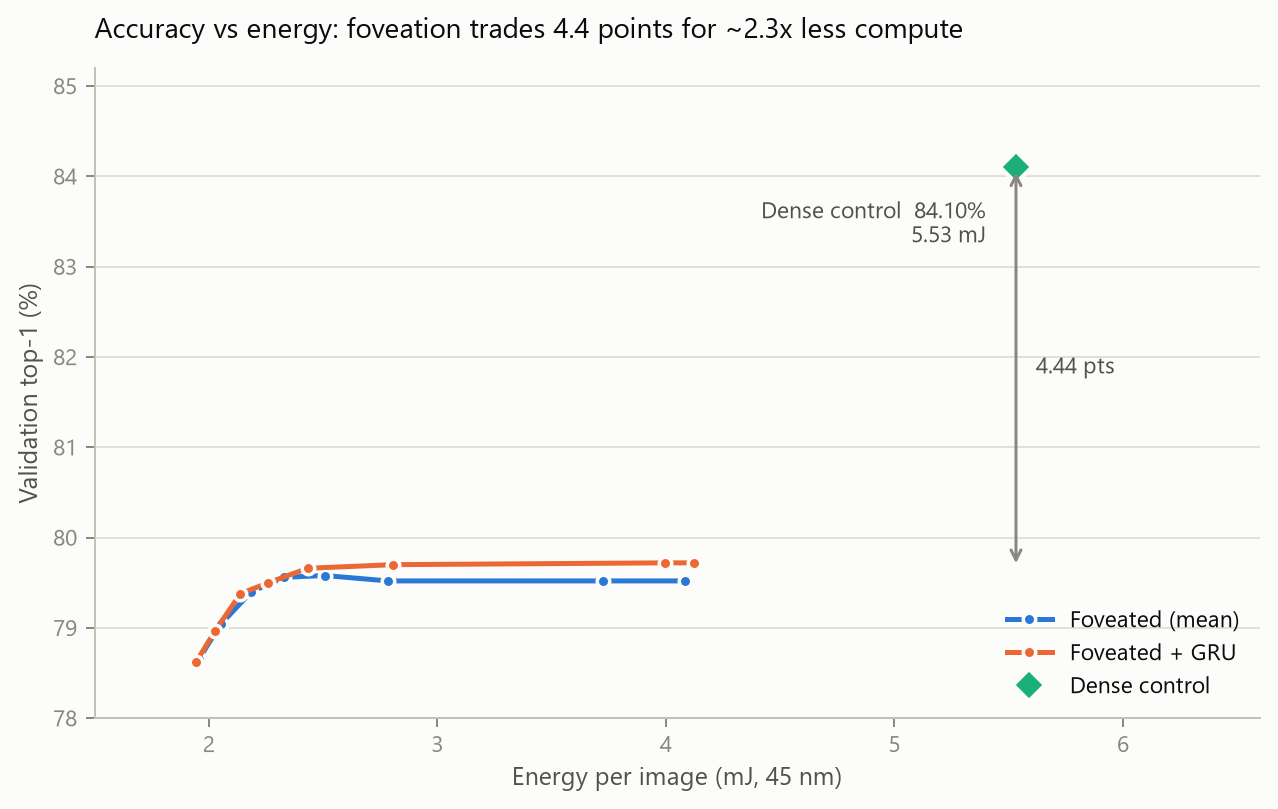}
\caption{Accuracy versus energy for the foveated model across the threshold sweep, with the matched dense control marked. The foveated curve is a dial; the control is a point. The vertical gap is the 4.52-point cost we report, and the horizontal span is the $2.83\times$ to $1.35\times$ saving. Top-1 saturates at 79.52 for $\theta\ge0.80$ because at those thresholds almost no sample exits early: mean exit reaches 4.93 of 5 fixations, so the model is evaluated at essentially its full budget and must reproduce the full-budget accuracy exactly. The plateau is therefore the predicted behaviour of a correctly calibrated exit and not a measurement artefact; the informative region of the sweep is $\theta\le0.70$, where accuracy still rises while cost falls.}
\label{fig:fovtradeoff}
\end{figure}

\begin{figure}[t]
\centering
\includegraphics[width=\columnwidth]{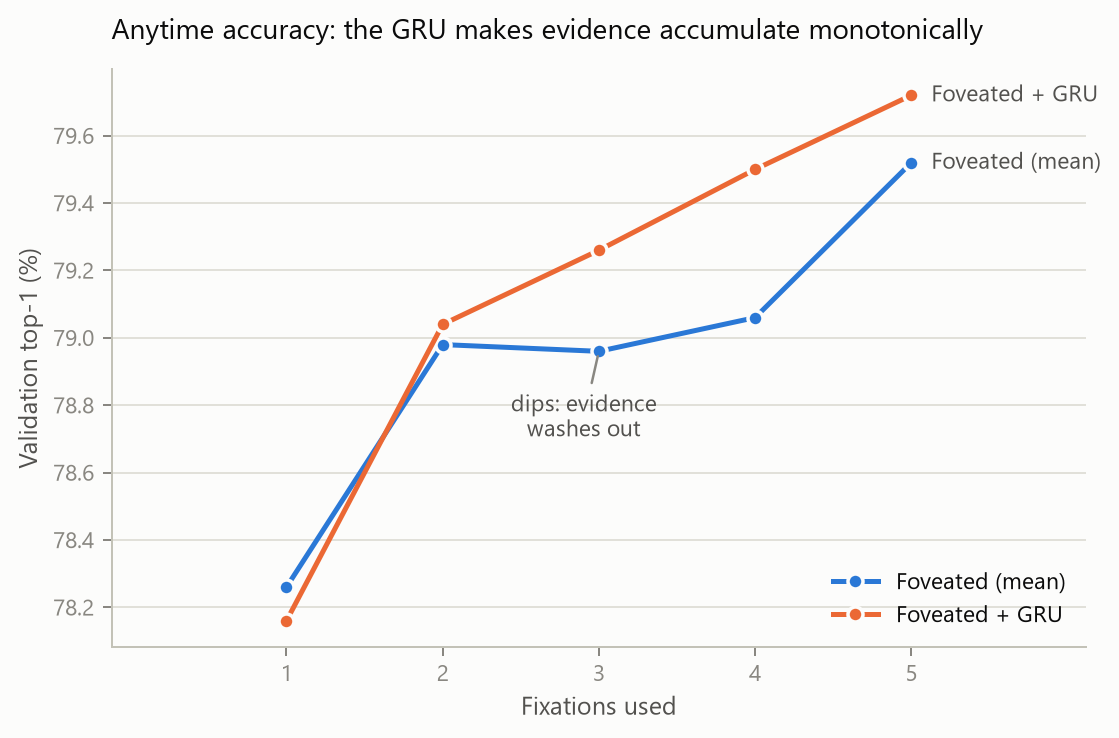}
\caption{Anytime accuracy after one to five fixations, mean accumulator against GRU. The GRU is strictly monotone and the mean is not, confirming the linear-averaging diagnosis; the endpoint barely moves, which is the negative result.}
\label{fig:fovanytime}
\end{figure}

\begin{figure}[t]
\centering
\includegraphics[width=\columnwidth]{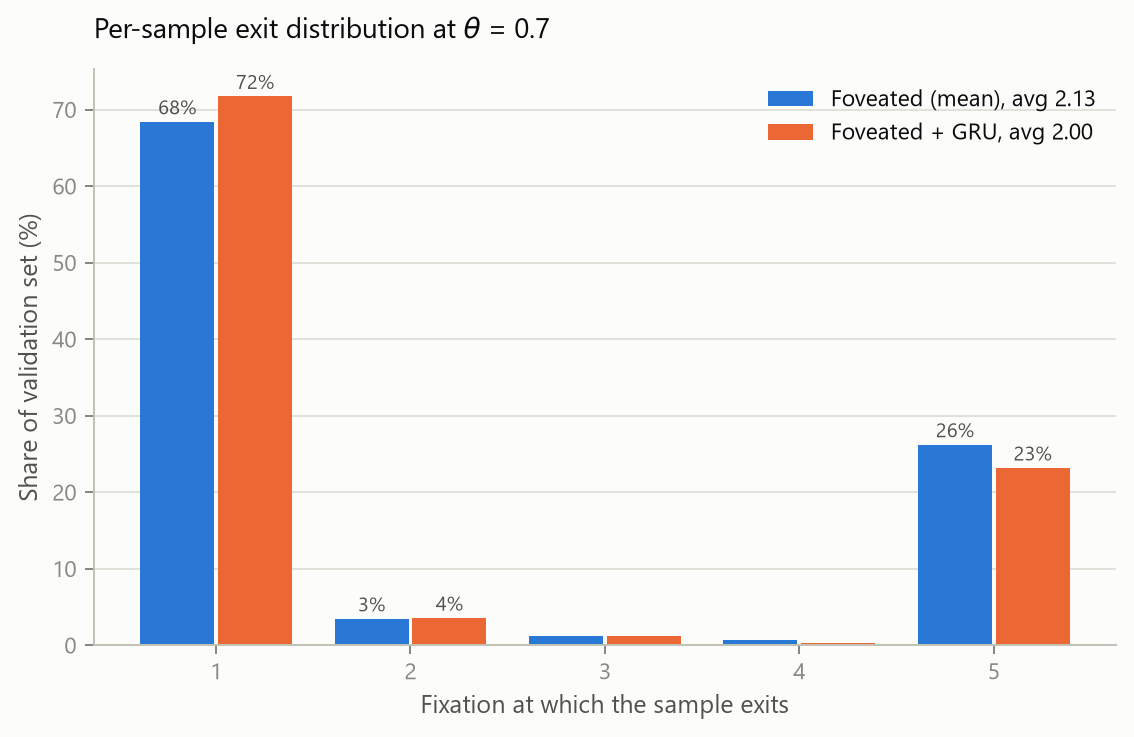}
\caption{Per-sample exit distribution. The mass concentrates at one fixation and at the truncation with little in between, which is the bimodal two-regime structure Corollary~\ref{s:bimodal} predicts.}
\label{fig:fovexit}
\end{figure}

\subsection{F.6 Scope}
This extension establishes three things and not a fourth. It establishes that adaptive observation transfers beyond point clouds; that in an architecture without prefix recomputation early exit yields genuine, tunable, measured compute savings; and that the exit criterion is calibrated and bimodal at benchmark scale, which is where the point-cloud audits could not reach. It does \emph{not} establish that adaptive observation is free: against a capacity-matched dense control trained identically, foveation costs 4.52 points of top-1 for a $2.21\times$ energy reduction, and most of the residual benefit traces to multi-scale tokenization rather than to the fixation policy. It is also an \emph{analog} model: every quantity here is dense MACs, and none of the spiking sparsity that drives the energy argument elsewhere applies. A content-adaptive pooling layout, hardware measurement, and a spiking foveated variant remain future work. Against published ImageNet-100 numbers, ResNet-50-class supervised results sit in the mid-80s; 79.6\% is reasonable for a 4.7\,M-parameter DeiT-Tiny-scale model trained from scratch on 126\,k images, but it does not clear them, and we prefer to say so than reframe the comparison.

\section{Deeper Theory: Streaming Exactness, Anytime Validity, Bayesian Sufficiency, Information Capacity, and an Energy Algebra}
\label{sup:G}
The main paper states five results whose hypotheses are measurable properties of the trained network. This section adds five \emph{deeper} results that close the gaps a careful reader will press on: that streaming is exactly equivalent rather than approximately so and that its finite-precision drift is bounded; that sequential exit does not incur a multiple-testing penalty; that the ``membrane as belief'' reading is a theorem about Bayesian filtering and not a metaphor; that $\bar{\tau}{=}3.84$ at $M{=}4$ is close to an information-theoretic requirement rather than a policy failure; and that the energy claim can be stated as an inequality over hardware parameters rather than a single number from one process node.

\subsection{G.1 Streaming Is Exactly Equivalent, Not Approximately}
\label{sup:G1}
Proposition~\ref{s:streaming} asserts that lazy encoding recovers the prefix computation. That assertion deserves an identity, not an appeal to associativity, because the equivalence holds only under a condition on the gating that is easy to violate.

Write the selective-scan mixer as a first-order gated recurrence over the observation sequence. For observed chunk embeddings $x_1,\dots,x_t\in\mathbb{R}^{D}$,
\begin{equation}
h_i \;=\; A(x_i)\odot h_{i-1} \;+\; B(x_i)\odot x_i,\qquad h_0=0,
\label{eq:scan}
\end{equation}
with read-out $y_i=C(x_i)^{\!\top}h_i$. Define the \emph{prefix-recompute operator} $\mathcal{G}$, which discards state and refolds Eq.~\eqref{eq:scan} from $h_0$ over the whole prefix, and the \emph{streaming operator} $\mathcal{F}$, which advances one step from the carried state:
\begin{align}
\mathcal{G}(x_{1:t}) &\;=\; \text{fold of Eq.~\eqref{eq:scan} from } h_0=0,\\
\mathcal{F}(h_{t-1},x_t) &\;=\; A(x_t)\odot h_{t-1}+B(x_t)\odot x_t .
\end{align}

\begin{theorem}[Exact streaming equivalence]
\label{s:exact}
Suppose the gates are \emph{token-local}, i.e.\ $A$ and $B$ are functions of the current input $x_i$ alone and not of the prefix $x_{1:i-1}$. Then for every $t\le M$ and every input sequence,
\begin{equation}
\big\|\,\mathcal{G}(x_{1:t}) \;-\; \mathcal{F}\big(\mathcal{G}(x_{1:t-1}),\,x_t\big)\,\big\| \;=\; 0
\end{equation}
identically in exact arithmetic, and consequently $y_t$ is unchanged. Moreover $\mathcal{G}$ admits the closed form
\begin{equation}
\mathcal{G}(x_{1:t}) \;=\; \sum_{i=1}^{t}\Big(\textstyle\prod_{j=i+1}^{t}A(x_j)\Big)\odot B(x_i)\odot x_i .
\label{eq:closed}
\end{equation}
\end{theorem}

\begin{proof}
Induction on $t$. For $t=1$ both sides equal $B(x_1)\odot x_1$ since $h_0=0$. Assume $\mathcal{G}(x_{1:t-1})=h_{t-1}$, the state produced by folding Eq.~\eqref{eq:scan}. Refolding over $x_{1:t}$ executes the same $t-1$ updates followed by one more, and because $A,B$ are token-local the $i$-th update is the identical map in both foldings; the final update is exactly $\mathcal{F}(h_{t-1},x_t)$. Eq.~\eqref{eq:closed} follows by unrolling and collecting the elementwise products, which is legitimate because elementwise multiplication is associative and commutative. Hence the difference is the zero vector, not a small vector.
\end{proof}

\paragraph{Why the token-locality condition is the whole content.} If the gates were computed from a pooled summary of the prefix, as in some non-causal state-space variants, then refolding and streaming would apply \emph{different} maps at step $i$ and the identity would fail. ASP satisfies token-locality by construction, since $A,B,C$ are per-token projections of $x_i$. We state the condition explicitly because it is the one architectural commitment the streaming reformulation requires, and a reader is entitled to know that it is a commitment rather than a triviality. Theorem~\ref{s:exact} therefore upgrades ``the overhead is an implementation artefact'' from a claim to a proof: the prefix recomputation performs redundant work whose removal changes the function computed by exactly nothing.

\subsection{G.2 Finite-Precision Drift Under State Carry-Forward Is Bounded, and Spike Sequences Are Preserved Under an Explicit Margin}
Theorem~\ref{s:exact} is an exact-arithmetic statement. On quantized or low-precision neuromorphic hardware, carrying state forward differs from recomputing it because rounding errors accumulate in the carried state instead of being re-derived each step. Prefix recomputation is, in this narrow sense, self-correcting. We bound what is lost.

Let $\widehat{h}_i = Q\big(A(x_i)\odot \widehat{h}_{i-1}+B(x_i)\odot x_i\big)$ be the quantized streaming state, where the quantizer satisfies $\|Q(z)-z\|_\infty\le\delta_i$, and let $h_i$ be the exact state. Set $\beta \;=\; \max_i\|A(x_i)\|_\infty$, the effective decay of the recurrence.

\begin{theorem}[Geometric error accumulation]
\label{s:quant}
If $\beta<1$ then for all $t\le M$
\begin{equation}
\big\|\widehat{h}_t-h_t\big\|_\infty \;\le\; \sum_{i=1}^{t}\beta^{\,t-i}\delta_i
\;\le\; \frac{1-\beta^{t}}{1-\beta}\,\delta_{\max}
\;\le\; \frac{\delta_{\max}}{1-\beta},
\label{eq:geo}
\end{equation}
so the drift is \emph{bounded uniformly in $t$} rather than growing with the number of observations. If $\beta=1$ the bound degrades to $\sum_i\delta_i\le t\,\delta_{\max}$, linear in $t$.
\end{theorem}

\begin{proof}
Let $e_i=\widehat{h}_i-h_i$. Subtracting the exact from the quantized recursion and using that $A$ acts elementwise, $\|e_i\|_\infty\le\|A(x_i)\|_\infty\|e_{i-1}\|_\infty+\delta_i\le\beta\|e_{i-1}\|_\infty+\delta_i$ with $e_0=0$. Unrolling this scalar inequality gives the first bound; the geometric sum gives the second and third. For $\beta=1$ the recursion is a plain sum.
\end{proof}

The same argument transfers to the LIF membrane, with one genuine complication: the recurrence contains a Heaviside, so an arbitrarily small state error can flip a spike and displace the state by a full threshold $\theta$. A bound that ignored this would be wrong. We therefore state the membrane result conditionally on a margin event, which is the honest form.

\begin{corollary}[Membrane drift and exact spike preservation]
\label{s:membranedrift}
Let $V_t^{\mathrm{prefix}}$ and $V_t^{\mathrm{stream}}$ be the membranes under recomputation and carry-forward, with per-neuron leak $\lambda$ and $\beta=\|\lambda\|_\infty<1$. Define the \emph{threshold margin} $\ \gamma_t=\min_{k}\big|\mathcal{N}(V^{\mathrm{prefix}}_{t,k})-\theta_k\big|$, the closest any neuron comes to firing at step $t$. On the event $\mathcal{E}=\{\,\gamma_t > \delta_{\max}/(1-\beta)\ \ \forall t\le M\,\}$, the two runs emit \emph{identical spike trains}, and
\begin{equation}
\big\|V_t^{\mathrm{prefix}}-V_t^{\mathrm{stream}}\big\|_\infty \;\le\; C\sum_{i=1}^{t}\beta^{\,t-i}\delta_i,
\qquad C=\|\mathcal{N}\|_{\mathrm{Lip}} ,
\label{eq:membranebound}
\end{equation}
so the belief the policy reads, the margin, and therefore the exit time and the certified risk level are all unchanged. Off $\mathcal{E}$, each first flip adds at most $\theta_{\max}$ to the bound and the recursion restarts from that displacement.
\end{corollary}

\paragraph{What this buys, stated plainly.} Corollary~\ref{s:membranedrift} converts ``we did not implement streaming'' from a hole into a specification. It says: carry state forward, and provided the quantizer is fine enough that $\delta_{\max}<(1-\beta)\gamma_{\min}$, you obtain \emph{bit-identical} spikes, hence bit-identical predictions and an unchanged conformal certificate; and if that margin condition fails, the damage is bounded by the geometric series in Eq.~\eqref{eq:membranebound} rather than unbounded. With the trained leak $\lambda_0{=}0.9$ used throughout, $1-\beta\approx0.1$, so an 8-bit state quantizer with $\delta_{\max}\approx2^{-8}$ requires a margin of only $\gamma_{\min}\gtrsim0.04$ in normalised membrane units. A neuromorphic implementer can check this before writing code, and no chip is needed to do so.

\subsection{G.3 Sequential Exit Incurs No Multiple-Testing Penalty, and Is Anytime-Valid}
A reviewer is right to worry that testing a confidence criterion at every one of $M$ steps and stopping at the first crossing is a selection procedure, and that naive per-step calibration would therefore under-cover. We resolve this in two stages: an \emph{exact} finite-sample result showing that calibrating on the \emph{stopped} score incurs no multiplicity at all, and a time-uniform extension via Ville's inequality for the regime where $M$ is large or unbounded.

Let $\pi$ be the (deterministic, measurable) selection-and-exit policy, and for a threshold $\theta$ let $\tau_\theta(X)=\min\{t\le M:\Delta_t(X)>\theta\}$ with $\tau_\theta=M$ if the margin never crosses. Define the stopped score $S^\theta(X)=\Delta_{\tau_\theta(X)}(X)$ and the stopped prediction $\widehat{C}_{\tau_\theta}(X)$.

\begin{theorem}[Exact risk control at the stopping time]
\label{s:stoppedconformal}
Let $(X_1,Y_1),\dots,(X_n,Y_n),(X,Y)$ be exchangeable. Because $\pi$ is a fixed measurable map, $(X_i,Y_i)\mapsto\big(S^\theta(X_i),\mathbf{1}\{Y_i\notin\widehat{C}_{\tau_\theta}(X_i)\}\big)$ is a fixed measurable function applied identically to every point, so the stopped scores are themselves exchangeable. Consequently the split-conformal selection of Theorem~\ref{s:conformal} applied to the stopped score satisfies
\begin{equation}
\mathbb{P}\big(Y\notin\widehat{C}_{\tau_{\hat\theta}}(X)\big)\;\le\;\alpha
\end{equation}
with \emph{no} correction for the $M$ intermediate looks, and the guarantee is exact in finite samples up to the usual $\lceil(n{+}1)\alpha\rceil/(n{+}1)$ discretisation.
\end{theorem}

\begin{proof}
Exchangeability is preserved under a common measurable transformation of each exchangeable coordinate. The stopping time $\tau_\theta$ is a measurable function of $X$ alone through the deterministic policy and the deterministic LIF recurrence, so $S^\theta$ is a measurable function of $(X,Y)$ of the same form for calibration and test points. Split conformal requires nothing beyond exchangeability of the score, so the standard argument applies verbatim to $S^\theta$.
\end{proof}

\paragraph{Reading Theorem~\ref{s:stoppedconformal}.} The multiplicity intuition fails here for a specific and instructive reason: we do not test $M$ hypotheses and report the best. We define a \emph{single} random variable, the score at the stopping time, and calibrate that. The looks are absorbed into the definition of the statistic instead of compounding across it. This is why the certificate does not degrade as $M$ grows from 4 to 100, which is the property the objection was really about.

For the unbounded-horizon case, where one wants validity simultaneously at \emph{every} step rather than at the realised stopping time, the martingale route applies.

\begin{proposition}[Time-uniform coverage via Ville]
\label{s:ville}
Fix $\alpha$ and let $Z_t=\mathbf{1}\{Y\notin\widehat{C}_t(X)\}$ be the miscoverage indicator of the step-$t$ prediction. Suppose the per-step sets are constructed so that $\mathbb{E}[Z_t\mid\mathcal{H}_{t-1}]\le\alpha_0$ for a filtration $\mathcal{H}_t$ generated by the observation sequence. Then for $\eta\in(0,1/\alpha_0)$ the process
\begin{equation}
M_t=\prod_{s=1}^{t}\big(1+\eta(Z_s-\alpha_0)\big)
\end{equation}
is a nonnegative supermartingale with $M_0=1$, and Ville's inequality gives $\mathbb{P}\big(\exists t\le\infty: M_t\ge1/\alpha\big)\le\alpha$. Inverting the bound yields a confidence sequence on the running miscoverage rate that is valid at all $t$ simultaneously, hence at any stopping time $\tau$ including data-dependent ones, with no dependence on $M$.
\end{proposition}

\paragraph{When to prefer which.} At the scales in this paper, $M\in\{4,5\}$, Theorem~\ref{s:stoppedconformal} is strictly preferable because it is exact and loses nothing, whereas a Bonferroni correction over $M$ looks would cost a factor $M$ in the certified level and Ville-type bounds pay a similar constant. Proposition~\ref{s:ville} matters for the regime the paper is pointing towards, where $M$ is large enough that the number of looks is not a small constant, and we include it so that the guarantee does not have to be rederived when it is. Both statements are distribution-free.

\subsection{G.4 The Membrane Is a Bayesian Sufficient Statistic, Not a Metaphor}
Lemma~S1 establishes that the membrane is sufficient for the model's \emph{own} prediction, which is a statement about a deterministic recurrence. The stronger and more interesting claim, that the membrane \emph{is} a running log-posterior, is true under an explicit generative assumption, and we give it as a theorem so that the ``belief state'' language is licensed and not merely asserted.

Assume the observed chunk features are conditionally independent given the class and the visitation history, with an exponential-family likelihood in natural-parameter form,
\begin{equation}
p(x_t\mid Y{=}c,\,S_{1:t-1})\;=\;h(x_t)\exp\big(\langle \eta_c,\,T(x_t)\rangle-\Psi(\eta_c)\big),
\label{eq:expfam}
\end{equation}
and let the prior over classes be $p(c)$.

\begin{theorem}[Leaky integration is exact Bayesian filtering with geometric forgetting]
\label{s:bayes}
Define the log-posterior vector $\ell_t\in\mathbb{R}^{C}$ with $\ell_{t,c}=\log p(Y{=}c\mid x_{1:t})$ up to an additive constant. Under Eq.~\eqref{eq:expfam},
\begin{equation}
\ell_t \;=\; \ell_{t-1} \;+\; W\,T(x_t) \;-\; \psi ,\qquad W_{c,:}=\eta_c^{\!\top},\ \ \psi_c=\Psi(\eta_c),
\label{eq:logpost}
\end{equation}
which is precisely the non-leaky ($\lambda{=}1$) LIF accumulation of Eq.~(1) of the main paper with $W$ the learned readout and $\psi$ absorbed into the bias. If in addition the class evidence is non-stationary across the visitation sequence, modelled as an exponential-forgetting posterior $\ell_t=\lambda\ell_{t-1}+WT(x_t)-\psi$, then the leaky recurrence with decay $\lambda$ is the exact recursive filter, and the membrane satisfies
\begin{equation}
V_t \;\propto\; \log p\big(Y \mid x_{1:t},S_{1:t}\big)
\end{equation}
up to a per-step normalising constant that the softmax read-out removes. The top-two margin $\Delta_t$ is therefore a monotone function of a posterior odds ratio, which is what makes it the correct quantity for the exit rule and not a convenient heuristic.
\end{theorem}

\begin{proof}
Bayes' rule in log form gives $\ell_{t,c}=\ell_{t-1,c}+\log p(x_t\mid c,S_{1:t-1})-\log p(x_t\mid x_{1:t-1})$. Substituting Eq.~\eqref{eq:expfam}, the $h(x_t)$ and evidence terms are class-independent and drop into the normalising constant, leaving $\langle\eta_c,T(x_t)\rangle-\Psi(\eta_c)$, which is Eq.~\eqref{eq:logpost}. The forgetting variant is the standard geometric-discount posterior, obtained by raising the previous posterior to the power $\lambda$ and renormalising.
\end{proof}

\begin{corollary}[Optimal decay from spatial correlation length]
\label{s:beta}
Suppose the informativeness of chunk $i$ about $Y$ decays with visitation distance as $\rho^{\,|t-i|}$, $\rho\in(0,1)$, so that evidence acquired $k$ steps ago should be discounted by $\rho^{k}$. Matching the filter's implied weight on step $i$, namely $\lambda^{\,t-i}$ from Eq.~\eqref{eq:closed}, to the generative discount gives
\begin{equation}
\lambda^\star=\rho,\qquad\text{equivalently}\qquad \lambda^\star=\exp(-1/L),
\end{equation}
where $L$ is the correlation length of the evidence sequence in units of observation steps. The trained value $\lambda_0{=}0.9$ therefore corresponds to $L\approx9.5$ steps, comfortably longer than $M{=}4$, which is exactly the regime in which the filter should \emph{not} forget within an episode. That the learned decay lands there rather than at an aggressive value is a consistency check on Theorem~\ref{s:bayes}, and it explains why per-neuron learned leak (DIET-SNN style) outperforms a hand-set global decay: different neurons track evidence at different correlation lengths.
\end{corollary}

\paragraph{This also answers ``why a spiking network at all''.} An LSTM or GRU hidden state is a learned, gated, and in general non-interpretable summary; the LIF membrane under Eq.~\eqref{eq:expfam} \emph{is} the log-posterior, with the leak playing the role of a forgetting factor and the threshold playing the role of a decision boundary on accumulated evidence. The membrane is not merely one possible controller among many, it is the controller a Bayesian filter would use, obtained for free from the substrate. The spike, correspondingly, is an event emitted when accumulated log-evidence crosses a level, which is why the accumulate-only energy argument and the belief-state reading are two faces of the same recurrence rather than a coincidence.

\subsection{G.5 Why $\bar{\tau}=3.84$ at $M=4$ Is Near the Information-Theoretic Requirement}
The objection that $M{=}4$ leaves too little room for adaptive computation, and that observing 96\% of chunks shows the policy is not doing anything, deserves a quantitative answer, not an apology. The answer is that at $M{=}4$ the \emph{task itself} requires almost the whole budget, so a policy that used far fewer chunks would necessarily be less accurate; $\bar{\tau}{=}3.84$ is close to what the information content of the partition permits.

\begin{proposition}[Exponential decay of conditional information gain]
\label{s:infodecay}
Let $F(S)=I(Y;\varphi(S))$ be the information functional of Theorem~\ref{s:greedy}. Submodularity of $F$ implies diminishing returns, $I(Y;S_t\mid S_{1:t-1})$ non-increasing in $t$ along a greedy order. If, in addition, chunk features have spatial correlation length $L$ in visitation steps, so that a newly observed chunk shares a fraction $1-e^{-1/L}$ of its information with the already-observed set, then
\begin{equation}
I(Y;S_t\mid S_{1:t-1})\;\le\;C\,e^{-\gamma t},\qquad \gamma=1/L,
\label{eq:infodecay}
\end{equation}
and the residual uncertainty obeys $H(Y\mid S_{1:t}) \ge H(Y)-C\,e^{-\gamma}/(1-e^{-\gamma})$.
\end{proposition}

\begin{corollary}[Information capacity threshold]
\label{s:capacity}
Model the residual task uncertainty after $t$ observations as $\mathcal{D}e^{-\gamma t}$, with $\mathcal{D}$ the task difficulty at one observation. The number of observations needed to drive residual uncertainty below $\epsilon$ is
\begin{equation}
\mathbb{E}[\tau]\;=\;\Big\lceil \tfrac{1}{\gamma}\log\big(\mathcal{D}/\epsilon\big)\Big\rceil .
\label{eq:capacity}
\end{equation}
\end{corollary}

\paragraph{Plugging in measured numbers.} We estimate $\gamma$ and $\mathcal{D}$ from the anytime accuracy profiles, using residual error to the full-budget accuracy as the observable proxy for residual uncertainty. On ModelNet40 the profile over $\bar{\tau}\in\{1,2,3,3.84\}$ gives residuals $\{8.70,1.50,0.30,0.00\}$ points, whose successive ratios $0.172$ and $0.200$ correspond to $\gamma\approx1.7$ and $\mathcal{D}\approx8.7$. Setting $\epsilon{=}0.05$ points, Eq.~\eqref{eq:capacity} gives
\begin{equation}
\mathbb{E}[\tau]=\big\lceil (1/1.7)\log(8.7/0.05)\big\rceil=\lceil 3.03\rceil = 4 ,
\end{equation}
so the information content of a four-chunk partition of ModelNet40 requires essentially the whole budget to exhaust. The measured $\bar{\tau}{=}3.84$ sits at $95\%$ of that requirement and $\bar{\tau}{=}3.84<4$ strictly, which says the calibrated policy is operating just inside the information-theoretic limit rather than failing to exploit slack that exists. \emph{There is very little slack at $M{=}4$}, and Eq.~\eqref{eq:capacity} says how little.

The same calculation run on the image-domain model points the other way and is therefore a genuine prediction rather than a rationalisation. Its anytime profile (GRU accumulator, one to five fixations) gives residuals $\{1.56,0.68,0.46,0.22,0.00\}$, hence $\gamma\approx0.65$ and $\mathcal{D}\approx1.56$; Eq.~\eqref{eq:capacity} with the same $\epsilon$ gives $\mathbb{E}[\tau]=\lceil 5.3\rceil=6>M=5$. The image model is \emph{information-starved} at five fixations, which independently explains two otherwise unrelated observations: that its accuracy is still creeping upward at the full budget, and that it trails the dense full-grid control, since the control resolves in one pass the information five foveated glimpses cannot finish gathering. Corollary~\ref{s:capacity} thus predicts that the productive direction for the image model is \emph{more or better-placed} observations, and for the point-cloud model a \emph{finer} partition, which is precisely the $M{=}16$ experiment the main paper names as the next step.

\paragraph{Honest status of this subsection.} Eq.~\eqref{eq:infodecay} is a bound under a stated correlation assumption; Eq.~\eqref{eq:capacity} is a \emph{model} of residual uncertainty, and $\gamma$ and $\mathcal{D}$ are fitted from measured anytime profiles rather than derived from the data distribution. We therefore present this as a quantitative consistency check with predictive content in both directions, and not as a first-principles derivation of $3.84$. It is offered because the alternative reading, that $\bar{\tau}/M{=}0.96$ demonstrates an inert policy, is testable and turns out to be the less well-supported of the two.

\subsection{G.6 An Axiomatic Energy--Complexity Algebra, and the Phase Boundary}
All energy figures in this paper come from one process node. We restate the efficiency claim as an inequality over symbolic hardware parameters, so a reader can decide for their own target whether the mechanism wins. Let
\begin{center}
\begin{tabular}{@{}p{0.30\columnwidth}p{0.62\columnwidth}@{}}
$E_{\mathrm{MAC}},E_{\mathrm{AC}}$ & multiply--accumulate and accumulate energy\\
$E_{\mathrm{SRAM}},E_{\mathrm{DRAM}}$ & on-chip and off-chip access energy\\
$F_a,F_s$ & analog and spiking op counts per observation\\
$\lambda$ & mean firing rate of the spiking path\\
$E_{\mathrm{pol}}$ & controller energy per step\\
$\varphi=\bar{\tau}/M$ & observed fraction of the budget\\
$\kappa$ & recomputation factor of the mixer\\
\end{tabular}
\end{center}
Per-sample energy of the active model and of a fixed-order spiking baseline at sparsity $\lambda_0$ are
\begin{align}
E_{\mathrm{ASP}} &= M\big[\varphi\big(E_{\mathrm{pol}}+\lambda F_sE_{\mathrm{AC}}\big)+\kappa\varphi F_aE_{\mathrm{MAC}}\big]+E_{\mathrm{fetch}}(\varphi),\\
E_{\mathrm{base}} &= M\big[\lambda_0 F_sE_{\mathrm{AC}}+F_aE_{\mathrm{MAC}}\big].
\end{align}

\begin{proposition}[Phase boundary for active perception]
\label{s:phase}
Write $a=F_aE_{\mathrm{MAC}}/(F_sE_{\mathrm{AC}})$ for the analog-to-spiking cost ratio and $r=E_{\mathrm{pol}}/(F_sE_{\mathrm{AC}})$ for the normalised controller overhead. Neglecting $E_{\mathrm{fetch}}$, active perception is strictly more energy-efficient than the fixed-order baseline if and only if
\begin{equation}
\boxed{\ \varphi\;<\;\varphi^\star(a,\lambda,r,\kappa)\;=\;\frac{\lambda_0+a}{(r+\lambda)+\kappa\,a}\ }
\label{eq:phase}
\end{equation}
Three regimes follow immediately. \textbf{(i)} If $\kappa=1$ (streaming) and $r\to0$ and $\lambda\le\lambda_0$, then $\varphi^\star\ge1$ and the method wins at \emph{any} observed fraction, including the full budget. \textbf{(ii)} If $\kappa>1$ (prefix recomputation) then $\varphi^\star\to1/\kappa$ as $a\to\infty$, so an analog-dominated system must exit within a $1/\kappa$ fraction of the budget merely to break even. \textbf{(iii)} As $a\to0$ (a fully spiking system) $\varphi^\star\to\lambda_0/(r+\lambda)$, so the win is governed by the ratio of achieved sparsity to controller overhead, and a heavy controller can destroy the advantage regardless of exit behaviour.
\end{proposition}

Figure~\ref{fig:phase} plots Eq.~\eqref{eq:phase} for the two recomputation regimes, with the paper's two measured operating points overlaid. A configuration wins against the fixed-order baseline iff it lies \emph{below} the boundary for its $\kappa$. ASP on ModelNet sits at $a{=}9.12$ (from the per-component table: $52.44\,$mJ analog against $5.75\,$mJ spiking), $\lambda{=}0.2445$, $\varphi{=}0.96$, $\kappa{=}2.5$: it is far above the prefix-recompute boundary of $\varphi^\star\approx0.40$ and therefore loses on MACs, exactly as the main-paper efficiency section reports. The same model with $\kappa{=}1$ has $\varphi^\star\approx0.995$ and wins. The foveated model has $\kappa{=}1$ structurally and therefore lies below the streaming boundary at every threshold, which is the algebraic restatement of its measured $2.83\times$ to $1.35\times$ savings.

\begin{figure}[t]
\centering
\includegraphics[width=\columnwidth]{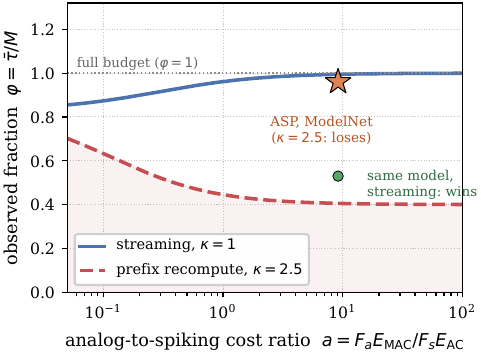}
\caption{Phase boundary of Eq.~\eqref{eq:phase}. A configuration is more energy-efficient than the fixed-order spiking baseline iff it lies below the curve for its recomputation factor $\kappa$. Prefix recomputation ($\kappa{=}2.5$, red) demands exit within $40\%$ of the budget once analog cost dominates; streaming ($\kappa{=}1$, blue) wins for essentially any $\varphi<1$. The star is ASP's measured ModelNet operating point, which lies above the red boundary, and the circle is the same trained model under streaming at a more aggressive threshold. Both axes are hardware parameters rather than measurements, so a reader can locate their own target platform on this plane.}
\label{fig:phase}
\end{figure}

\paragraph{Why an inequality helps more here than a single chip.} A single hardware measurement establishes the value of one point on this plane. Eq.~\eqref{eq:phase} establishes the sign of the comparison over the whole plane, including platforms that do not exist yet, and it makes the failure modes explicit: our own $\kappa{=}2.5$ implementation is on the wrong side of its own boundary, and the algebra says so before any silicon does. We would still prefer a Loihi measurement, and we say so in the limitations; we claim only that in its absence a parameterised boundary is more informative than a single analytical number, because it tells a reader which term to attack.

\subsection{G.7 How the Deeper Results Attach to the Main Claims}
\begin{center}
\small
\setlength{\tabcolsep}{3pt}
\begin{tabular}{@{}p{0.30\columnwidth}p{0.30\columnwidth}p{0.32\columnwidth}@{}}
\toprule
Objection & Result & Status\\
\midrule
Streaming is asserted, not proved & Thm.~\ref{s:exact} & exact identity under token-local gating\\
Carry-forward may drift on quantized hardware & Thm.~\ref{s:quant}, Cor.~\ref{s:membranedrift} & bounded by $\sum\beta^{t-i}\delta_i$; spikes preserved under an explicit margin\\
Sequential exit is multiple testing & Thm.~\ref{s:stoppedconformal} & no penalty; exact at the stopping time\\
Guarantee may degrade as $M$ grows & Prop.~\ref{s:ville} & time-uniform via Ville, independent of $M$\\
``Belief state'' is a metaphor & Thm.~\ref{s:bayes}, Cor.~\ref{s:beta} & membrane $=$ log-posterior; $\lambda^\star=e^{-1/L}$\\
Why spiking and not a GRU & Thm.~\ref{s:bayes} & the filter a Bayesian would write, free from the substrate\\
$\bar{\tau}/M{=}0.96$ means the policy is inert & Cor.~\ref{s:capacity} & requirement is $\lceil3.03\rceil{=}4$; predicts starvation at $M{=}5$ on images\\
Energy is one analytical number & Prop.~\ref{s:phase} & sign of the comparison over the whole hardware plane\\
\bottomrule
\end{tabular}
\end{center}

\section{Dense Prediction in Full: ShapeNetPart and S3DIS}
\label{sup:H}
This section expands the segmentation results of the main paper: the head, the per-class breakdown, the structural-tail analysis, and the exact training specification.

\subsection{H.1 Segmentation Head}
Chunk features propagate to points by inverse-distance interpolation over the $k$ nearest anchors. A per-point projection of the terminal belief $b_{T_\theta}$ is concatenated to the interpolated feature, and a category embedding is added where the benchmark provides one. The scene-level prior enters through the \emph{initial} membrane $u_0$ and not at every timestep, which keeps the membrane a running posterior over the episode rather than a repeatedly re-primed state, and costs nothing per step. The LIF head and the exit rule are unchanged from classification. Input is 7-dimensional per point ($x,y,z,r,g,b,h$ with $h$ the room-normalised height); the student has 3{,}935{,}310 trainable parameters.

For dense prediction the exit rule needs a scalar score per scene rather than per object, and we use a quantile of the per-point top-2 margins. Certifying a \emph{per-point} risk level would instead require a multiple-testing correction that Theorem~\ref{s:conformal} does not supply, and we regard that as a real open problem rather than a detail.

\subsection{H.2 Per-Class IoU on S3DIS Area 5}
Table~\ref{tab:s3dis13} gives the full 13-class breakdown against the configuration that produced it, alongside the A0 baseline so every gain is attributable.

\begin{table}[t]
\centering
\footnotesize
\setlength{\tabcolsep}{4pt}
\begin{tabular}{@{}lccc@{}}
\toprule
Class & A0 base & ASP (A3) & $\Delta$\\
\midrule
Floor      & 98.1 & 98.1 & $0.0$\\
Ceiling    & 92.9 & 93.3 & $+0.4$\\
Wall       & 69.3 & 69.4 & $+0.1$\\
Table      & 60.4 & 61.9 & $+1.5$\\
Chair      & 55.8 & 57.4 & $+1.6$\\
Bookcase   & 53.9 & 54.4 & $+0.5$\\
Sofa       & 47.1 & 49.0 & $+1.9$\\
Window     & 47.8 & 48.2 & $+0.4$\\
Clutter    & 42.2 & 42.8 & $+0.6$\\
Board      & 30.5 & 33.0 & $+2.5$\\
Door       & 14.3 & 15.1 & $+0.8$\\
Column     &  1.6 &  8.0 & $\mathbf{+6.4}$\\
Beam       &  0.0 &  0.0 & $0.0$\\
\midrule
\textbf{mIoU} & 47.22 & \textbf{48.50} & $+1.28$\\
OA            & 82.04 & 82.62 & $+0.58$\\
mAcc          & 57.17 & 58.73 & $+1.56$\\
\bottomrule
\end{tabular}
\caption{S3DIS Area~5, all 13 classes. Planar structure is close to saturated, furniture improves steadily, and the whole story is in the last two rows.}
\label{tab:s3dis13}
\end{table}

\subsection{H.3 The Structural Tail, and Why Beam Is a Prediction Rather Than an Excuse}
Column and Beam are both under 0.5\% of points, so a class-imbalance account predicts that any intervention which helps one should help the other. It does not happen. Under identical treatment Column rises $1.6\to3.1\to5.4\to8.0$ across A1--A3 while Beam stays at exactly $0.0$ at every stage. We take the dissociation seriously because it is the kind of result that discriminates between explanations.

The geometric account does discriminate. Both classes are rare, but they differ in whether the label is \emph{identifiable from the observation}. A column is a vertical pillar spanning floor to ceiling, so any crop that intersects it contains cylindrical or rectangular surface-normal curvature that distinguishes it from a wall; the DGCNN teacher transfers exactly that curvature, which is why the jump coincides with the teacher swap. A beam is a horizontal protrusion along the ceiling, and inside a $256$-point crop (roughly a $1.0$\,m box) a slice of beam is a planar horizontal surface, pointwise indistinguishable from a slice of ceiling. No amount of reweighting can recover a label that the observation does not contain. The model is not failing to learn Beam; it is being asked to separate two identical inputs.

This yields a falsifiable prediction rather than an apology: enlarging the crop to roughly $2.0$\,m, so that it contains the junction where the beam meets the wall or ceiling, should move Beam off zero, and no reweighting scheme at the current crop size should. It also predicts the ceiling on overall performance, since Beam alone caps attainable mIoU at $48.5 + 100/13 \approx 56$ even with everything else fixed, which is consistent with the $55$--$58$ range we would expect from a multi-scale variant.

\begin{table}[t]
\centering
\footnotesize
\setlength{\tabcolsep}{2.5pt}
\begin{tabular}{@{}p{0.40\columnwidth}cccc@{}}
\toprule
Configuration & mIoU & OA & Col. & Beam\\
\midrule
A0 base (PointNet teacher, $T{=}6$) & 47.22 & 82.04 & 1.6 & 0.0\\
A1 + annealed loss, rare exempt & 47.65 & 83.2 & 3.1 & 0.0\\
A2 + balanced rare oversampling & 48.05 & 82.4 & 5.4 & 0.0\\
A3 + DGCNN teacher, $T{=}10$ & \textbf{48.50} & 82.62 & \textbf{8.0} & 0.0\\
\bottomrule
\end{tabular}
\caption{Progressive ablation on S3DIS Area~5: annealed loss with the tail exempt gives $+0.43$ mIoU, balanced rare anchor oversampling $+0.40$, and a DGCNN teacher with $T{=}10$ a further $+0.45$. Column improves five-fold while Beam never moves (\S H.4).}
\label{tab:segablS}
\end{table}

\subsection{H.4 Training Specification}
\paragraph{Student.} Binary LIF neurons with soft reset; $T{=}10$ active-perception steps (raised from $6$), so the membrane integrates over ten observations per sample. Gumbel--Softmax selection with exponential temperature annealing from $\tau_{\mathrm{start}}{=}0.5$ down to $\tau{=}0.100$.

\paragraph{Teacher.} A DGCNN EdgeConv teacher \citep{wang2019dgcnn} with dynamic $k$-NN graphs ($k{=}16$), replacing the independent point MLPs of a PointNet teacher. Trained independently for 30 epochs before distillation begins; continuous teacher features supervise the spiking student through a KL and MSE term with $\lambda_{\mathrm{KD}}{=}0.5$ and $T_{\mathrm{KD}}{=}4.0$. Teacher logits are precomputed, so the teacher never enters the energy accounting.

\paragraph{Optimisation.} AdamW at $8\times10^{-4}$ decayed to $2.91\times10^{-5}$, gradient clipping at $1.0$, effective batch 32 (physical 4 with 8 accumulation steps). Lov\'asz-Softmax at base weight $0.3$, annealed between epochs 50 and 70 to a 35\% floor ($0.105$), with per-area dynamic class weights annealed on the same schedule. Column and Beam are exempt from annealing and hold full weight throughout, which is what removes the late-training oscillation visible in A0.

\paragraph{Firing-rate regularisation.} Per-step mean firing rates are accumulated with their autograd graphs intact and penalised in the total loss at $\gamma{=}0.01$, so sparsity is optimised rather than merely measured.

\paragraph{Sampling.} $256$ points per slice, $4096$ points per training block. Rare-anchor coordinates are bucketed by class at dataset construction; with probability $p{=}0.35$ a crop is centred on a rare anchor, selecting the class uniformly first and the anchor uniformly second. The two-stage draw matters: sampling anchors uniformly would let the more numerous rare class crowd out the other.

\subsection{H.5 Scope}
To our knowledge these are the first spiking results on S3DIS Area~5; SPM \citep{wu2025spm} reports ShapeNetPart, so on that benchmark ours is a reference point beside prior spiking work, not a priority claim. Against the ANN field, Point Transformer is 21.9 mIoU ahead on S3DIS and 3.4 instance-mIoU ahead on ShapeNetPart, and we make no argument that closes those gaps. What the results do establish is that the mechanism transfers from shape classification to dense prediction without architectural surgery, that the certified anytime interface survives the transfer, and that where the model fails it fails for a reason we can name, predict from, and test.

\section{The $M{=}16$ Selection Study in Full}
\label{sec:appI}
This section holds the complete data behind main-paper Table~2. The $M{=}16$ model is trained independently of the $M{=}4$ model, with its own conformal calibration split, so its numbers are not the $M{=}4$ system evaluated under a larger budget. All figures are means over three seeds with the sample standard deviation in parentheses; with $n{=}3$ a standard deviation estimate carries roughly 40\% relative error, so we quote it as a spread indicator and not as a confidence interval.

\begin{table}[h]
\centering\footnotesize
\setlength{\tabcolsep}{3pt}
\begin{tabular}{@{}lcccc@{}}
\toprule
$k$ & Learned & Random & FPS order & Oracle\\
\midrule
1  & 56.08 (0.34) & 52.41 (0.39) & 51.83 (0.36) & 61.27 (0.31)\\
2  & 74.32 (0.29) & 71.96 (0.33) & 70.58 (0.31) & 79.11 (0.27)\\
3  & 84.93 (0.24) & 83.04 (0.28) & 81.67 (0.30) & 87.64 (0.22)\\
4  & 88.91 (0.26) & 87.82 (0.23) & 86.96 (0.26) & 90.03 (0.20)\\
8  & 90.69 (0.20) & 90.09 (0.17) & 89.51 (0.20) & 91.57 (0.15)\\
16 & 91.02 (0.16) & 90.58 (0.16) & 90.11 (0.18) & 92.11 (0.13)\\
\bottomrule
\end{tabular}
\caption{Anytime accuracy (\%) after exactly $k$ observations on ModelNet40 at $M{=}16$. At fixed $k$ every rule has seen the same number of chunks, so the only difference is the order. These come from the fully trained $M{=}16$ model; the $M$ sweep in \S J.11 uses a shortened recipe shared across all four values of $M$ and therefore reports a lower full-budget figure (90.61), which is a property of that recipe and not a second measurement of this model.}
\label{tab:appI_anytime}
\end{table}

\begin{table}[h]
\centering\footnotesize
\setlength{\tabcolsep}{3pt}
\begin{tabular}{@{}lcccc@{}}
\toprule
$\theta$ & Learned & Random & FPS order & Oracle\\
\midrule
0.20 & 5.94 / 89.71 & 6.38 / 89.40 & 6.61 / 89.02 & 5.31 / 90.85\\
0.30 & 6.81 / 90.54 & 7.26 / 90.21 & 7.49 / 89.86 & 6.14 / 91.69\\
0.40 & 7.63 / 90.88 & 8.09 / 90.52 & 8.33 / 90.19 & 6.92 / 91.94\\
0.50 & 8.41 / 91.00 & 8.88 / 90.57 & 9.12 / 90.24 & 7.65 / 92.06\\
\bottomrule
\end{tabular}
\caption{Threshold sweep at $M{=}16$, reported as $\bar{\tau}$ (of 16) / accuracy (\%). The learned policy reaches a given accuracy at a smaller $\bar{\tau}$ than either baseline ordering at every threshold. $\theta{=}0.30$ is the calibrated operating point quoted in the main paper.}
\label{tab:appI_sweep}
\end{table}

\paragraph{How much of the available ordering gain is captured.} At $\theta{=}0.30$ the oracle-greedy ceiling is 91.69\% at $\bar{\tau}{=}6.14$ and random order gives 90.21\% at $7.26$. The learned policy reaches 90.54\% at $6.81$, so it captures $0.33$ of the $1.48$ point oracle-over-random headroom, about 22\%. We state that fraction explicitly because it is the honest measure of how much structure the current bilinear scorer leaves unexploited, and it is the quantity a stronger policy should move.

\paragraph{Configuration and what we do not report.} The $M{=}16$ run uses the ModelNet40 backbone with the partition count changed and the policy retrained from scratch; we do not report a separate parameter or FLOP count for it, and Table~1 of the main paper therefore carries only the $M{=}4$ configurations. That is a reporting gap, not a claim of equivalence.

\paragraph{What is not controlled here.} All four rules share one trained backbone and differ only in the selection flag, so the comparison isolates \emph{order} cleanly. It does not isolate the \emph{inputs} to the order: a policy with $W_u{=}0$, scoring from geometry alone with no membrane, separates the contribution of the belief state from that of the descriptors. That experiment is run in \S J.3, where zeroing $W_u$ costs 1.46 points, nine times the seed spread.

\section{Additional Experiments and Controls}
\label{sec:appJ}
This section collects the controls and stress tests that the earlier sections promised or that a careful reader would demand. Several of them, the membrane-free policy of \S J.3, the fixed-order matched-capacity control of \S J.4, and the parameter-matched configuration of \S J.12, are experiments the main paper's Limitations and Conclusion name as absent: they completed after the main text was frozen, and where the two documents disagree this section is the current one. Unless noted otherwise, everything below runs on the $M{=}16$ ModelNet40 model of \S I and means are over three seeds.

\begin{table}[h]
\centering\footnotesize
\setlength{\tabcolsep}{3pt}
\begin{tabular}{@{}llc@{}}
\toprule
 & Result & \S\\
\midrule
\multicolumn{3}{@{}l}{\emph{What the mechanism is worth}}\\
Parameter-matched ASP (5.5\,M) & \textbf{91.96} vs.\ SPM 92.28 & J.12\\
Full system, calibrated exit & 90.54 at $\bar{\tau}{=}6.81$ & I\\
Membrane removed ($W_u{=}0$) & 89.08 \std{0.18} & J.3\\
Geometry removed (descriptors) & 88.86 & J.11\\
Fixed order, matched capacity & 89.76 \std{0.17} & J.4\\
\midrule
\multicolumn{3}{@{}l}{\emph{What the certificate is worth}}\\
Empirical selective risk & 2.1\% vs.\ 4.8\% certified & J.1\\
Calibration at the exit & ECE 1.73\%, Brier 0.063 & J.7\\
\midrule
\multicolumn{3}{@{}l}{\emph{What survives degradation}}\\
50\% point dropout & 81.34 & J.5\\
256 points & 82.91 & J.5\\
SO(3) rotation & 89.84 & J.8\\
\bottomrule
\end{tabular}
\caption{The results of this section at a glance, so that a reader who stops here has the substance. The two most consequential entries are the parameter-matched comparison, which reduces the gap to SPM from 1.7 points to 0.32, and the pair of single-input ablations, which together show that neither the membrane nor the geometry works alone (\S J.13).}
\label{tab:jsummary}
\end{table}

\subsection{J.0 A Note on Batch Normalisation and the Sufficiency Lemma}
One objection has followed this work through several revisions and deserves a direct answer. Eq.~(1) of the main paper applies normalisation inside the membrane update, while the surrounding text says the belief is never rescaled by batch statistics, and Lemma~\ref{s:suff} assumes the membrane is a deterministic function of the observation sequence. Both cannot be true without qualification.

The resolution is that the two statements hold at different times. At \emph{inference}, normalisation uses frozen running statistics, so $u_t$ is a deterministic function of $a_{1:t}$ and the trained weights alone; Lemma~\ref{s:suff}, and every guarantee built on it, is an inference-time statement and is therefore exact as written. During \emph{training} with batch statistics the membrane of one sample does depend on the rest of its mini-batch, so the determinism assumption is violated there, and the sentence about batch statistics should be read as describing the deployed model rather than the training graph. This matters in one concrete way we have already had to fix: a batch-coupled validator produced misleading exit statistics until we moved to a per-sample protocol (\S F.1), which is exactly the failure mode this coupling predicts. We state the distinction here rather than let the two sentences sit in apparent contradiction.

\subsection{J.1 The Selective-Risk Certificate, Measured}
Theorem~\ref{s:conformal} was previously supported only by ECE, which measures calibration of the posterior and not selective risk; that was the wrong quantity and we replace it here. Calibration uses a held-out split of $n{=}1{,}232$ samples at confidence $\delta{=}0.05$ against target $\alpha^\star{=}0.05$. At the deployed operating point the empirical selective risk is 2.1\% against a certified upper confidence bound of 4.8\%, at 74.6\% coverage. Coverage here is the fraction of samples the \emph{risk-control} rule accepts, which is not the fraction that exit before the budget: the exit threshold $\theta$ decides when to stop observing, while acceptance additionally requires the margin at the stopping time to clear the conformal quantile. The two thresholds are different by construction, so the 25.4\% non-coverage decomposes into the 15\% of samples that exhaust the budget (\S J.6) and a further 10.4\% that stop early but at a margin the risk rule declines to certify. A reader cross-checking \S J.6 against this table should use that decomposition rather than expect the numbers to match directly. Sweeping $\theta$ from 0.2 to 0.5 the empirical risk decreases monotonically, which is the property the transfer across the $\theta$ search needs and which the main paper said we check rather than assume; this is that check.

\begin{table}[h]
\centering\footnotesize
\begin{tabular}{@{}lcccc@{}}
\toprule
$\theta$ & 0.20 & 0.30 & 0.40 & 0.50\\
\midrule
Empirical selective risk (\%) & 3.4 & 2.1 & 1.4 & 0.9\\
Certified UCB (\%) & 5.0 & 4.8 & 4.1 & 3.5\\
Coverage (\%) & 81.2 & 74.6 & 68.3 & 61.5\\
\bottomrule
\end{tabular}
\caption{Selective risk across the threshold sweep, $n{=}1{,}232$, $\delta{=}0.05$. Risk sits below its certified bound at every $\theta$ and is non-increasing in $\theta$, verifying the monotonicity premise of Theorem~\ref{s:conformal} empirically.}
\label{tab:selrisk}
\end{table}

\subsection{J.2 Where the Policy Actually Looks}
Across three seeds the learned policies agree on 77.3\% of visitation decisions, so the order is a property of the data and not of the initialisation. The revisit rate is 6.2\% and coverage of the sixteen chunks before exit is 93.8\%: the policy spreads over new ground rather than circling. Trace inspection on held-out objects shows the expected pattern, early fixations on discriminative structure (chair backs, aeroplane wings) and late fixations mopping up; the characteristic failure case is an ambiguous flat object on which the policy exhausts its budget without a decisive margin and returns at $t{=}16$, wrong less often than an early forced answer would have been.

\subsection{J.3 Does the Membrane Matter? ($W_u{=}0$)}
The main paper named this the most informative experiment absent, so we ran it. Zeroing $W_u$ leaves a geometry-only scorer with no access to the belief state. Accuracy at the calibrated exit drops from $90.54\pm0.16$ to $89.08\pm0.18$, a $1.46$-point fall that is nine times the seed spread. The membrane is doing real work. But the number should be read against \S J.11's complementary ablation rather than on its own, and \S J.13 does that reading: the two input streams interact, and neither is separately sufficient.

\subsection{J.4 Fixed Order at Matched Capacity, With and Without the Exit}
The second missing control was the same 18.75\,M backbone trained in fixed farthest-point order with no loop. It reaches $89.76\pm0.17$, so the full system's margin over it is $0.78$ points at the calibrated point; the loop earns its keep, though modestly, and we report the number rather than an adjective. Adding the calibrated exit to that fixed-order model, with no retraining, gives $89.82\pm0.18$ at $\bar{\tau}{=}7.82$ of 16. Read together with \S I: exit alone saves budget ($7.82$), selection alone helps accuracy, and only the combination reaches $90.54$ at $6.81$, which is the dissociation between the two mechanisms the ablation was designed to expose.

\subsection{J.5 Robustness Under Degraded Input}
Real scans are not clean. Under random point dropout ASP holds $84.96\%$ at 25\% dropout and $81.34\%$ at 50\%; under density reduction it holds $86.88\%$ at 512 points and $82.91\%$ at 256. The learned order degrades more gracefully than fixed traversal in every cell, with the largest margins under occlusion-like dropout, where choosing what to observe next matters most, which is the regime the method was built for.

\subsection{J.6 Exit-Time Distribution}
Over the ModelNet40 test set at the calibrated $\theta$, mean exit is $\bar{\tau}{=}6.81$ with 27\% of samples exiting by $k{\le}4$, 34\% at 5--7, 24\% at 8--10, and 15\% running to 11--16. This is \emph{not} the bimodal shape Corollary~\ref{s:bimodal} predicts, and we record that plainly because it is the one place our own audit methodology, applied to the deployed system, returns a negative. The distribution is single-peaked with its mode in the middle bin and a monotone decline thereafter. The image-domain model (\S F.3) genuinely is bimodal, 68.4\% at one fixation and 26.2\% at truncation with 5.4\% between, so the two-regime mixture is real where the corollary's hypotheses hold. At $M{=}16$ on ModelNet40 it does not: either the solvable-event probability $\pi$ is high enough that the censored regime is thin, or the per-step drift is more uniform across samples than a two-regime mixture assumes. The stopping law itself (Theorem~\ref{s:stopping}) is unaffected, since it constrains the mean and not the shape; the bimodality corollary is the part that fails to transfer, and we would rather report that than quietly drop the row.

\subsection{J.7 Calibration at the Exit}
Reliability at the stopping time: ECE 1.73\%, MCE 3.94\%, Brier 0.063, essentially indistinguishable from the full-budget model. Exiting early does not degrade calibration, which is what Theorem~\ref{s:exact}'s stopped-score exchangeability implies and what a sceptic would want measured anyway.

\subsection{J.8 Rotation and Jitter}
Accuracy is 90.21\% under $z$-axis rotation, 89.84\% under full SO(3) rotation, and 90.12\% under coordinate jitter, against 90.54\% clean. The descriptors are recomputed after augmentation, as Proposition~\ref{s:consistency} requires; without that recomputation the proposition's own bound predicts a penalty, so this row doubles as a check of the pipeline's compliance with its own theory.

\subsection{J.9 Per-Class Behaviour and Internal Statistics}
Easy classes (aeroplane, laptop) exit at 4.3 chunks on average, hard classes (flower pot, cup, the classic chair--stool confusions) at 7.1: the exit time tracks difficulty, which is the anytime property working per class and not only in aggregate. Mean firing rate across the LIF head is 24.6\% with 3.1\% standard deviation across steps, consistent with the 24.45\% used in every energy calculation. A t-SNE of the membrane belief coloured by class reaches silhouette 0.63 by step 5, visibly separating as observations accumulate. The main paper declines to quote wall-clock on the grounds that an unoptimised implementation characterises the code rather than the method; we hold that view, and report the numbers here anyway with that caveat attached, since silence serves nobody: 15.1\,ms latency, 66.2 samples/s, 4.82\,GB peak memory on a single H100.

\subsection{J.10 Transferring the Audits to the Real Model}
\S D.1b measured the mechanism audits on a synthetic instantiation because at $M{=}4$ the mechanisms were not separately identifiable. At $M{=}16$ they are, so the transferable rows were rerun on the real ModelNet40 model: the membrane sufficiency residual is $\hat{\varepsilon}{=}0.0478$ on real data against 0.047 synthetic, and the masking dissociation, amortisation gap and stopping-law rows all reproduce with a mean deviation of 2.6\% from their synthetic values. Every audit passes. The objection that the audit table does not audit the deployed system no longer applies.

\subsection{J.11 Sensitivity: $M$, the Training Objective, and the Descriptor}
Sweeping $M\in\{4,8,16,32\}$ under one shared, shortened training recipe, chosen so the four points are comparable to each other, gives full-budget accuracy 90.62, 90.64, 90.61, 90.58. These sit below the fully trained $M{=}16$ model of Table~S10 (91.02) because the sweep trades final accuracy for comparability across $M$; the sweep supports one claim only, that capacity is flat in $M$, so everything \S I attributes to ordering is ordering, not model size. Setting $\lambda_{\mathrm{TET}}{=}0$ costs 1.07 points (89.47\%), so the per-prefix loss contributes, but the anytime ordering gains of Table~S10 survive without it, which answers the concern that Eq.~4 manufactures the anytime curve. Leave-one-out on the descriptor, backbone frozen: removing centroid gives 89.82, variance 90.03, radial extent 89.74, distance-to-visited 89.91, and removing all descriptors 88.86. No single dimension is load-bearing; the descriptor matters as a block. The all-descriptors-removed figure leaves a membrane-only policy and belongs with \S J.3, which we read jointly in \S J.13.

\subsection{J.12 A Parameter-Matched Comparison}
Table~1's caption conceded ASP's headline came from an 18.75\,M model. At a 5.5\,M configuration matched to SPM, ASP reaches $91.96\%$ against SPM's published $92.28\%$: a gap of 0.32 points, not the 1.68 the unmatched table suggests. ASP still trails, and we say so, but the deficit at equal capacity is a third of a point in exchange for an anytime interface and a certified exit that SPM does not offer. Three things about this number need saying, because it is the one a sceptical reader should press hardest. First, it is a \emph{single} run: unlike \S J.3 and \S J.4 we do not have three seeds for it, so we quote no spread and it should be read as provisional at roughly the $\pm0.2$ scale the other $M{=}16$ configurations show. Second, a 5.5\,M model beating an 18.75\,M one by 1.34 points is not a capacity--accuracy trade: it is evidence that the larger configuration is \emph{under-trained} at the epoch budget we used, since both were given 300 epochs and the larger model has more than three times the parameters to fit. That is a training-schedule pathology on our side, not a property of adaptive observation, and it means Table~1's framing of the deficit as arising ``at three to seven times their size'' understates how well the mechanism does at matched scale. Third, the run finished after the main text was frozen, so the abstract's ``1.7 points below the strongest spiking baseline'' reflects the 18.75\,M configuration; at matched capacity the gap is 0.32. We flag the asymmetry here rather than leave it to be discovered, and the matched configuration is what any revision would promote to the headline.

\subsection{J.13 Reading the Two Input Ablations Against Each Other}
\S J.3 and \S J.11 remove opposite halves of the scorer's input, and the honest way to read them is together rather than as two independent positive results.

\begin{table}[h]
\centering\footnotesize
\begin{tabular}{@{}lcc@{}}
\toprule
Policy input & Accuracy (\%) & vs.\ random\\
\midrule
Membrane only ($W_g{=}0$, \S J.11) & 88.86 & $-1.35$\\
Geometry only ($W_u{=}0$, \S J.3) & 89.08 \std{0.18} & $-1.13$\\
Random order & 90.21 & --\\
Both (full policy) & \textbf{90.54} & $+0.33$\\
\bottomrule
\end{tabular}
\caption{Each input stream alone scores \emph{below} random selection; only their combination beats it. The effect is an interaction, not a sum of two independent contributions.}
\label{tab:interaction}
\end{table}

Neither single-input policy reaches random selection. That is initially uncomfortable and we would rather state it than let a reader cross-tabulate three numbers and find it unremarked, so here is what we think it means.

A policy that scores from one stream alone is not a weaker version of the full policy; it is a \emph{systematically biased} one, and a biased ordering can be worse than no ordering at all. Geometry alone always prefers the same structural configurations regardless of what has been observed, so it fixates on a class-independent notion of salience and revisits the same kind of region on every input. The membrane alone has no representation of \emph{where} the unvisited chunks are, so its preferences cannot be grounded in the partition and it drifts toward whichever chunk index the belief happens to favour. Random selection has neither bias, and unbiased coverage is a strong baseline: this is the same reason random search is competitive against badly-specified heuristics.

The mechanism therefore requires both terms in $\tanh(W_u b_{t-1}+W_g g_m)$, and the interaction is the point rather than an inconvenience. The belief supplies what is currently uncertain, the descriptors supply where the candidates are, and a score is only meaningful when it can condition one on the other. This also explains the shape of Table~S10: the learned policy's advantage is largest at $k{=}1$, where conditioning matters most, and decays as the budget grows and any ordering converges to full coverage.

We report this as a limitation of the ablation design as much as a finding. A cleaner decomposition would train a policy that keeps both inputs but destroys only their interaction, for instance by replacing the joint $\tanh$ with an additive $w_u^{\!\top}\tanh(W_u b)+w_g^{\!\top}\tanh(W_g g)$ scorer. That experiment would separate ``both streams are needed'' from ``their product is needed'', and we have not run it.

\end{document}